\documentclass{article}
\usepackage{amsmath}
\usepackage{enumitem}
\usepackage{amsthm}
\theoremstyle{plain}
\newtheorem{theorem}{Theorem}[section]
\newtheorem{proposition}[theorem]{Proposition}
\newtheorem{corollary}[theorem]{Corollary}

\theoremstyle{definition}

\theoremstyle{remark}

\usepackage[main, final]{neurips_2026}

\usepackage[utf8]{inputenc} 
\usepackage[T1]{fontenc}    
\usepackage{hyperref}       
\usepackage{url}            
\usepackage{booktabs}       
\usepackage{amsfonts}       
\usepackage{nicefrac}       
\usepackage{microtype}      
\usepackage{xcolor}         
\usepackage{graphicx}
\usepackage{cleveref}
\usepackage{algorithm}
\usepackage{algpseudocode}
\usepackage{float}
\usepackage{multirow}
\usepackage{graphicx}
\usepackage{adjustbox}
\usepackage{arydshln}
\usepackage{inlinelabel}

\title{COMET: Decoupled Distillation, Routing, and Capacity Control for Task-Agnostic Continual Vision--Language Learning}

\author{Seungyoon Woo, Keonhee Park, Gunhee Kim \\
Seoul National University \\
{\small \texttt{\{seungyoon.woo, keonhee.park\}@vision.snu.ac.kr, gunhee@snu.ac.kr}} \\
}

\begin{document}

\maketitle

\begin{abstract}
Vision--language models (VLMs) such as CLIP retain zero-shot recognition ability, but their open-world use often requires task-agnostic continual learning: tasks arrive sequentially, past data are unavailable, task identity is unknown, and the backbone cannot be retrained.
We argue that this setting fails VLMs through three coupled pressures beyond stability--plasticity tradeoff: (1) modal-consistency drift, (2) representation-space interference, and (3) per-task capacity shortfall.
We propose COMET (\textbf{CO}llaborative \textbf{M}ixture-of-\textbf{E}xpert \textbf{T}ransfer) to resolve these three issues respectively as follows.
A frozen teacher provides only a feature reference on a shared image--text pool; Tri-Affinity Distillation preserves image--image, text--text, and image--text geometry; and Student-Aware Cross-Expert Distillation aligns new experts with compatible prior experts so routing errors degrade gracefully.
We design COMET to keep capacity and inference separate: a contextual bandit keeps or merges full per-task experts using post-training signals, while a reconstruction router selects sparse experts at test time without teacher semantics.
On the X-TAIL benchmark, we show that each proposed components indeed improves the continual learning performance.
COMET shows that continual vision--language adaptation can preserve multimodal geometry, coordinate accumulated experts, and allocate capacity jointly when feature reference, expert shaping, routing, and capacity control are separated by construction.
\end{abstract}

\section{Introduction}

Recent Vision--language models~\citep{liu2023visual, liu2024improved, chen2024internvl, wang2024qwen2} such as CLIP~\citep{radford2021learning}, trained on web-scale image--text pairs, can recognize objects they were never explicitly taught; this capacity is known as zero-shot transfer.
However, in open-world deployment, new tasks can arrive sequentially  (chest X-rays~\citep{samir2023covid}, then satellite imagery~\citep{helber2019eurosat}, then microscopy~\citep{iwata2022sem}), and the model is expected to be \textit{continually learned}  under three desiderata: no access to stored examples from earlier tasks, no knowledge of task identity at inference time, and no budget to retrain the backbone for each new task.
Recent task-agnostic image-classification benchmarks such as X-TAIL\citep{xu2024x-tail} pin down this scenario: during each training phase only the current task's data are available, and at inference the model picks a label from a generalized candidate space containing all classes seen so far and even unseen classes. 
Because full-backbone updates can be too expensive, the community has focused on \emph{parameter-efficient} continual learning~\citep{houlsby2019parameter, li2021prefix, park2024pre}: only small task-specific modules (adapters~\citep{yu2024moe-adapter, luo2025lada}, prompts, or expert networks) are updated, leaving the backbone fixed.
This lighter paradigm can still fail if the stream becomes long, heterogeneous, and open-ended.

We argue that this failure does not stem from only a stability--plasticity tradeoff~\citep{kirkpatrick2016forgetting, mccloskey1989catastrophic}, but rather from the simultaneous failure of three axes, each with its own characteristic signature, whereas prior methods typically address only one axis at a time.
First, \emph{modal-consistency drift} is the divergence between a model's pre-training and downstream alignments as adaptation accumulates.
A CLIP-style model~\citep{radford2021learning} carries two alignments: one from pre-training (behind zero-shot transfer), the other from downstream adaptation accumulated across the stream.
Protecting one could silently damage the other; for example, a model well adapted to microscopy may begin to misrank everyday objects it had previously recognized.
Second, \emph{representation-space interference} is the incompatibility among independently-trained experts under imperfect routing.
Since each is trained in isolation, a routing mistakenly sends an ambiguous input into a module that has never been shaped to handle it, and errors concentrate near expert boundaries.
Third, \emph{per-task capacity shortfall} is the inability of a small module to represent feature geometries far from pre-training.
A small adapter lacks the representational power for the feature geometry of medical imaging, microscopy, and remote sensing.

Existing method families each make this three-way trade explicit. Weight-interpolation methods (e.g., WiSE-FT~\citep{wortsman2022robust-ft}, ZSCL~\citep{zheng2023zscl}) mitigate drift but starve distant domains of capacity, because no convex combination of pre-trained and ask-adapted weights can introduce new feature directions. Mixture-of-experts~\citep{shazeer2017outrageously} with lightweight adapters (e.g., MoE-Adapters~\citep{yu2024moe-adapter}, LADA~\citep{luo2025lada}) add capacity but never coordinate the experts they accumulate, because each adapter is trained in isolation, leaving no shared geometry~\cite{pascanu2014saddle_point} for the router to degrade into. Full fine-tuning can recover capacity but lose both alignments, because rewriting the backbone may erase the pre-training geometry itself.
Each axis is relieved at the expense of the other two.

We show that the three pressures can be resolved jointly, but only if their mechanisms are kept strictly disjoint.
Entangling them is why prior methods~\citep{wortsman2022robust-ft, zheng2023zscl, yu2024snd, xu2024x-tail} trade one failure for another.
When the teacher model is also used as the router, its inherent semantic biases can distort capacity allocation: tasks appear similar from its perspective due to pre-training, even if their downstream representation geometries differ substantially.
Similarly, a merging rule~\citep{frankle2020linear_conncetivity, neyshabur2020transferred_transfer} that reuses the adaptation loss conflates compatibility with plasticity; it merges two experts whenever they fit the current task equally well, ignoring possible long-term overlap.
An adapter that simultaneously absorbs new tasks and corrects past ones eventually fails at both, as its capacity is divided between forward learning and backward retention, with each undermining the other.
Across these methods, a common pattern emerges: multiple roles are coupled into a single mechanism, leading to systemic collapse.

COMET (\textbf{CO}llaborative \textbf{M}ixture-of-\textbf{E}xpert \textbf{T}ransfer)which enforces this separation by assigning each role to a dedicated component.
A frozen vision--language teacher serves only as a task-invariant feature reference: its affinity structure on a shared pool of image--text pairs anchors distillation, and its frozen image encoder fixes the feature space in which routers operate.
Neither the router nor the merging criterion uses its predictions.
A lightweight contextual-bandit policy~\citep{beygelzimer2011contextual_bandit} serves only as the capacity controller, deciding at each task boundary whether to allocate a new expert or consolidate into an existing one.
Using a bandit~\citep{lattimore2020bandit} (rather than a fixed threshold) allows allocation to track the novelty of the incoming stream.
It reads only post-training signals that never enter the adaptation loss, such as accuracy on a held-out subset of the reference pool and weight-space similarity to the nearest prior.
Its reward is defined as the post-action accuracy on a held-out reference pool. 
Distillation~\citep{hinton2015distilling, hinton2015kd} shapes only the experts.
Since allocation does not inherit teacher semantics and compatibility corrections do not entangle with capacity control, the retained expert count grows sublinearly in the task count whenever cumulative novelty is sublinear or stays within a bounded region of the representation space.
Each new task's feature geometry tends to lie close to the span of at least one prior expert on the shared reference pool, so consolidation aligns with novelty. 
This is an empirical regularity rather than a worst-case guarantee: a stream that keeps entering new regions can still force linear growth, even when each per-task shift is bounded.

Three mechanisms realize this decomposition, one per failure axis, each breaking one coupling.
\textbf{Tri-Affinity Distillation (TAD)} addresses modal-consistency drift by matching image--image, text--text, and image--text similarity structures between student and frozen teacher on the shared pool, anchoring the student to pre-training geometry.
\textbf{Student-Aware Cross-Expert Distillation (SACED)} addresses routing-induced interference by aligning each new expert with compatible prior experts on the same pool, so a misrouted sample lands in a still-coherent module, not foreign ones.
Since the router is imperfect, we design for graceful degradation rather than assuming perfect routing.
\textbf{Full per-task experts with bandit-controlled merging} address per-task capacity shortfall: COMET trains a full compact expert per task instead of a shared-backbone adapter~\citep{latif2024shared_adapter}, and the bandit merges when consolidation is justified.
The merge operates on already-frozen candidates at task boundaries, avoiding the optimization-time interpolation of weight-ensembling.

Three lines of evidence from X-TAIL experiments validate this decomposition. 
Axis-wise ablations confirm individual necessity with each mechanism's removal restores only its own failure mode: drift, interference, or shortfall:
\begin{itemize}
    \item \textbf{Modal geometry is preserved.} TAD improves Transfer accuracy by 4.7\%p over the no-distillation variant, and a 2.6$\times$ smaller student retains the teacher's zero-shot ability while improving downstream accuracy.
    \item \textbf{Routing failures are absorbed, not amplified.} SACED raises accuracy on misrouted samples by 10.8\%p without exemplar replay, and 52.5\% of misrouted samples are still classified correctly.
    \item \textbf{Capacity matches distant domains at sublinear storage cost.} Full experts under bandit-controlled merging improve Transfer by 18.7\%p over the strongest baseline on a four-dataset benchmark spanning medical, microscopy, and histopathology domains. The pool retains roughly 5-8 experts after 10 X-TAIL tasks, illustrating sublinear growth in practice.
\end{itemize}

These results reframe continual vision--language adaptation: not as a single trade-off to balance, but as three problems to solve jointly: preserving multimodal geometry, coordinating accumulated modules, and allocating sufficient capacity.  
To the best of our knowledge, this is the first work that establishes training-time role disjointness, across the feature reference, the capacity controller, and the per-expert shaping objective, as the central design principle of a parameter-efficient continual VLM framework.
Disjointness is not emergent, it must be enforced by construction.

\section{Related Work}
\paragraph{Continual Learning Settings.}
Conventional continual learning~\citep{hadsell2020embracing, de2021continual, li2018lwf, rebuffi2017icarl} is typically studied under Task-Incremental Learning (TIL)~\citep{li2018lwf} and Class-Incremental Learning (CIL)~\citep{rebuffi2017icarl} settings, where models are sequentially trained on tasks and evaluated with or without access to task identities at inference, respectively. 
To relax the assumption that all tasks originate from a single domain, the Multi-domain Task-Incremental Learning (MTIL)~\citep{zheng2023zscl} considers tasks drawn from diverse domains with domain-specific class distributions. 
More recently, Cross-domain Task-Agnostic Incremental Learning (X-TAIL)~\citep{xu2024x-tail} further removes task identity information and requires models to perform inference across both seen and unseen domains. 
The X-TAIL is a particularly challenging benchmark in VLM-based continual learning, as it demands not only the retention of previously learned knowledge but also the preservation of the zero-shot ability.

\paragraph{CLIP-based Continual Learning Methods.}
Recent continual learning studies~\citep{wortsman2022robust-ft,zheng2023zscl,yu2024moe-adapter,xu2024x-tail,luo2025lada} have widely adopted CLIP~\citep{radford2021learning}, leveraging its strong representations and zero-shot ability.
Conventional approaches fall into two categories: fine-tuning and Parameter-Efficient Fine-Tuning (PEFT)~\citep{houlsby2019parameter,hu2022lora,li2021prefix}.
Fine-tuning methods~\citep{wortsman2022robust-ft,zheng2023zscl} directly update the CLIP.
WISE-FT~\citep{wortsman2022robust-ft} fine-tunes CLIP and interpolates the updated weights with the pre-trained weights to balance adaptation and zero-shot transfer.
ZSCL~\citep{zheng2023zscl} further stabilizes the interpolation via knowledge distillation from both the pre-trained model and a reference dataset, reducing unnecessary shift and preserving zero-shot capability.
PEFT methods~\citep{yu2024moe-adapter,xu2024x-tail,luo2025lada}, by contrast, keep the backbone frozen and introduce lightweight modules.
MoE-Adapters~\citep{yu2024moe-adapter} augment CLIP with Mixture-of-Experts adapters and routers, while retaining the CLIP for zero-shot ability on unseen tasks.
RAIL~\citep{xu2024x-tail} employs a ridge regression--based adapter paired with a training-free fusion module to preserve zero-shot ability.
LADA~\citep{luo2025lada} appends class-wise memory vectors to the image encoder via label-specific adapters, enabling efficient adaptation without auxiliary parameter selection at inference.
Unlike existing methods that fine-tune a CLIP or attach lightweight modules to the pre-trained CLIP, our approach treats each CLIP as a task-specific expert and extends the MoE~\citep{shazeer2017outrageously} to preserve sufficient plasticity.

\begin{figure}[t]
    \centering
    \includegraphics[width=0.98\linewidth]{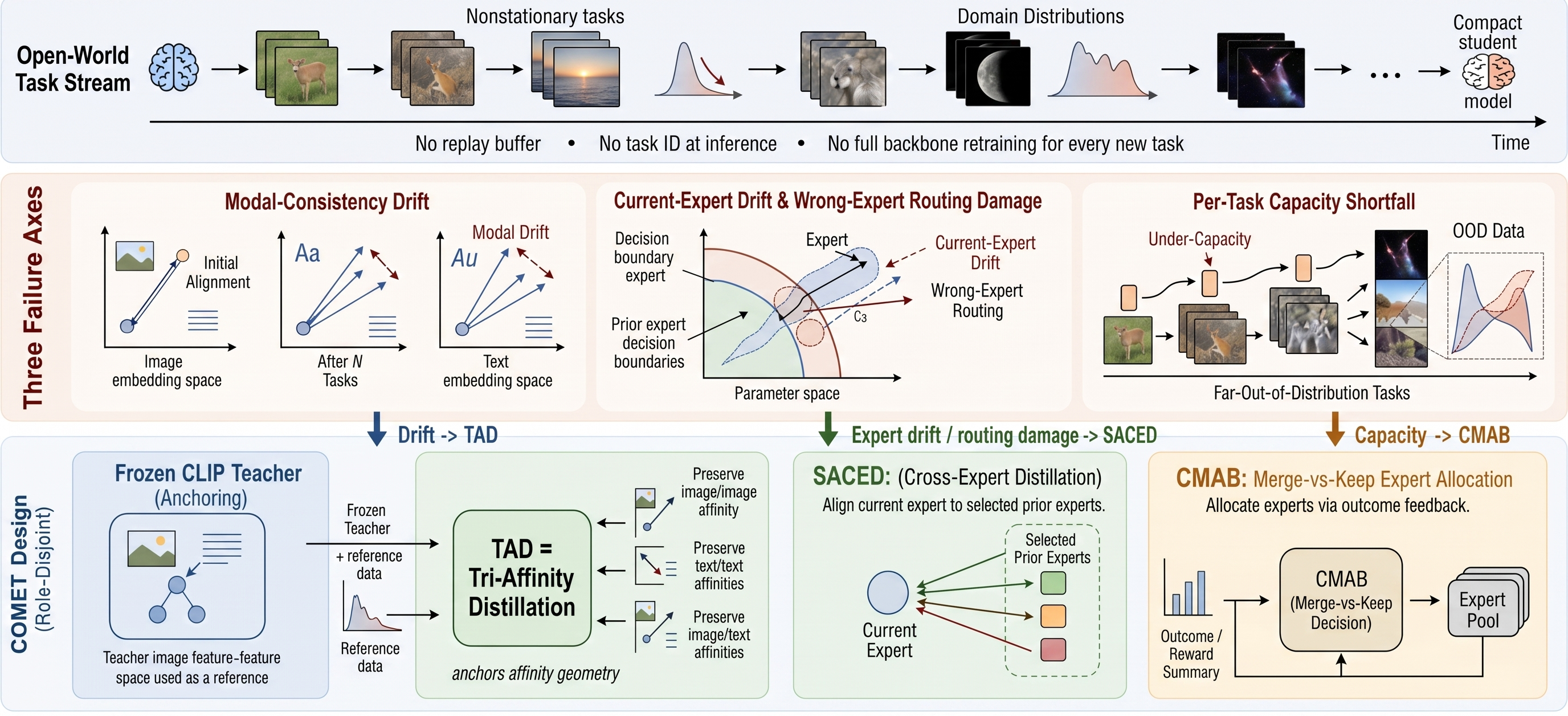}
    \vspace{-2mm}
    \caption{\textbf{Overview of COMET.}
    Targeting task-agnostic continual vision--language learning without replay or full-backbone retraining.
    It decomposes three failure modes: modal drift, routing-induced expert interference, and capacity shortfall into three role-disjoint remedies: TAD for CLIP-geometry preservation, SACED for cross-expert compatibility, and CMAB for adaptive expert allocation.}
    \label{fig:main}
    \vspace{-2mm}
\end{figure}

\section{Preliminary}

\textbf{The Continual Learning Setting.}
We consider the X-TAIL~\citep{xu2024x-tail}, where a model sequentially learns tasks 
$\mathcal{T} = \{\mathcal{T}_1, \mathcal{T}_2, \dots, \mathcal{T}_K\}$.
Each task $\mathcal{T}_k$ has a labeled dataset 
$\mathcal{D}_k = \{(x_i^{(k)}, c_i^{(k)})\}_{i=1}^{N_k}$ 
with class labels $c_i^{(k)} \in \mathcal{C}^{(k)}$.
At step $k$, the model trains only on $\mathcal{D}_k$ without access to previous task data.
Let $\mathcal{C}_L^{(k)} = \bigcup_{i=1}^{k} \mathcal{C}^{(i)}$ denote the set of seen class labels after learning task $k$, and let $\mathcal{C}_U$ denote unseen class labels that are not observed during supervised continual training.
Following the X-TAIL evaluation protocol, inference is task-agnostic and is performed over the generalized label space $\mathcal{C}_N^{(k)} = \mathcal{C}_L^{(k)} \cup \mathcal{C}_U$.
At test time, the task identity is unknown and the model selects a label from both previously seen and unseen classes.

\section{Approach: COMET}
\Cref{fig:main} maps each failure axis to a disjoint mechanism.
Given a task stream, COMET incrementally builds a pool of task-specific CLIP experts and isolates each role by construction so that geometry preservation, expert compatibility, and capacity control never share an objective.
Accepted experts receive no further gradient updates; they change only when the capacity controller merges a later expert into them.
COMET uses a frozen vision--language teacher only as a feature reference: the teacher anchors distillation on a shared image--text reference pool $\mathcal{D}_{\mathrm{ref}}$, and its frozen image encoder supplies the feature space in which routers operate.
Neither the router nor the merge rule uses its predictions.
A contextual bandit is used only as the capacity controller, deciding at task boundaries whether the newly trained expert should remain separate or be merged with a compatible prior expert (\S\ref{sec:cmab}).
The expert itself is shaped by two distillation objectives: Tri-Affinity Distillation (TAD; \S\ref{sec:tri-affinity}) preserves pre-trained multimodal geometry, and Student-Aware Cross-Expert Distillation (SACED; \S\ref{sec:saced}) aligns the new expert with compatible prior experts so routing errors degrade gracefully.
At inference, a lightweight reconstruction-based router selects sparse top-$k$ experts for each input without consulting teacher semantics (\S\ref{sec:routing}).
The remainder of this section follows the data path: anchor the new expert to multimodal geometry, coordinate it with compatible prior experts, decide whether the pool should grow, and route test inputs by reconstruction fit.

Each task follows the same execution order.
COMET initializes a new expert from the shared CLIP initialization, trains it on the current task data with task supervision plus the two distillation terms, evaluates post-training compatibility signals on $\mathcal{D}_{\mathrm{ref}}$, and only then invokes the bandit to keep or merge the expert.
The per-task objective is
$\mathcal{L}_{k} = \mathcal{L}_{\mathrm{task}} + \mathcal{L}_{\mathrm{TAD}} + \lambda_2\,\mathcal{L}_{\mathrm{SACED}}$,
where $\mathcal{L}_{\mathrm{task}}$ is the supervised loss on $\mathcal{D}_k$, the two distillation losses use $\mathcal{D}_{\mathrm{ref}}$, and $\lambda_2$ balances SACED against the geometry-anchoring TAD term (within TAD itself, $\lambda_1$ controls the cross-modal channel.
When no compatible prior expert exists, the SACED expert set is empty and the SACED term is zero.
Only after this optimization step does the capacity controller act.
This ordering prevents the teacher, the router, and the merge policy from becoming hidden parts of the adaptation objective.

\subsection{Tri-Affinity Distillation (TAD)}
\label{sec:tri-affinity}

TAD targets modal-consistency drift by preserving the relation structure of the pre-trained CLIP space~\citep{radford2021clip}.
Sequential adaptation on non-i.i.d.\ tasks can move image features, text features, and their cross-modal alignment in different directions, eroding both transferability and zero-shot generalization~\citep{wortsman2022robust-ft,zheng2023zscl}.
Because task labels vary across the stream, logit distillation gives an unstable anchor; TAD instead uses the frozen CLIP teacher only as a label-agnostic feature reference on $\mathcal{D}_{\mathrm{ref}}$ and matches the three CLIP affinity channels---image--image, text--text, and image--text---between the student expert and the teacher.
These three constraints preserve the geometry needed for transfer while leaving enough freedom for task-specific adaptation.

\textbf{Intra-Modality Distillation.}
We apply intra-modality distillation~\citep{hinton2015distilling} separately to image--image and text--text relations so the student keeps the teacher's within-modality geometry.
For a reference mini-batch of $B$ image--text pairs from $\mathcal{D}_{\mathrm{ref}}$, let $S_{\mathrm{I}}, T_{\mathrm{I}} \in \mathbb{R}^{B\times d}$ denote the $\ell_2$-normalized image embeddings produced by the student and teacher, respectively, and define $S_{\mathrm{T}}, T_{\mathrm{T}} \in \mathbb{R}^{B\times d}$ analogously for text.
The image-affinity loss minimizes the discrepancy between the student and teacher pairwise-similarity matrices, $\mathcal{L}_{\mathrm{I}}=
\left\|
S_{\mathrm{I}}S_{\mathrm{I}}^{\top}
-
T_{\mathrm{I}}T_{\mathrm{I}}^{\top}
\right\|_F^2$,
and the text-affinity loss $\mathcal{L}_{\mathrm{T}}$ takes the same form on $(S_{\mathrm{T}},T_{\mathrm{T}})$.
We do not introduce an additional affinity temperature, since a scalar divisor inside this squared loss would only rescale the objective and is therefore equivalent to adjusting the relative loss weight.
Matching these within-modality structures inherits the pre-trained CLIP's modality-specific geometry and reduces representation drift during adaptation.

\textbf{Inter-Modality Distillation.}
To maintain the shared vision--language embedding space, we align image-to-text and text-to-image similarity distributions between the student expert and the frozen pre-trained CLIP teacher.
For each modality pair $(a,b) \in \{(\mathrm{I},\mathrm{T}),(\mathrm{T},\mathrm{I})\}$ and model $m \in \{S,T\}$, we compute scaled cross-modal logits $Z^{m}_{a\rightarrow b}=\alpha\, E^{m}_{a}(E^{m}_{b})^{\top},\;m \in \{S,T\}$,
where $\alpha$ is the logit scaling factor, and
$E^{S}_{\mathrm{I}}=S_{\mathrm{I}}$, $E^{S}_{\mathrm{T}}=S_{\mathrm{T}}$,
$E^{T}_{\mathrm{I}}=T_{\mathrm{I}}$, and $E^{T}_{\mathrm{T}}=T_{\mathrm{T}}$.
Converting each row to a probability distribution with temperature $\tau_c$, the inter-modality loss uses a teacher-to-student forward KL divergence:
\begin{equation}
\small
\begin{aligned}
\mathcal{L}_{\mathrm{Cross}}
&=
\frac{\tau_c^{2}}{2B}
\sum_{(a,b)\in\{(\mathrm{I},\mathrm{T}),(\mathrm{T},\mathrm{I})\}}
\sum_{i=1}^{B}
\sum_{\ell=1}^{B}
P^{T}_{a\rightarrow b,\,i\ell}
\log
\frac{
P^{T}_{a\rightarrow b,\,i\ell}
}{
P^{S}_{a\rightarrow b,\,i\ell}
},
\end{aligned}
\label{eq:cross_modal_loss}
\end{equation}
where
$P^{m}_{a\rightarrow b,\,i\cdot}
=
\mathrm{softmax}(Z^{m}_{a\rightarrow b,\,i\cdot}/\tau_c)$,
$\operatorname{sg}(\cdot)$ denotes stop-gradient, and $\ell$ ranges over the within-batch candidates.
The outer summation averages the two cross-modal directions; the KL divergence itself remains asymmetric, with the frozen teacher distribution serving as the target.

The final loss of TAD is $\mathcal{L}_{\mathrm{TAD}}=\mathcal{L}_{\mathrm{I}}+\mathcal{L}_{\mathrm{T}}+\lambda_1\mathcal{L}_{\mathrm{Cross}}$,
where the two intra-modality losses and the bidirectional cross-modal distillation loss keep all three affinity channels tied to the same teacher geometry. Consequently, TAD constrains drift without forcing the student to imitate task-specific teacher logits.

\subsection{Student-Aware Cross-Expert Distillation (SACED)}
\label{sec:saced}

SACED targets representation-space incompatibility among independently trained experts under imperfect routing.
Freezing prior experts prevents direct overwriting, but a new expert optimized only for the current task can still become geometrically incompatible with its neighbors, so a misrouted input may land in a module with no shared representational frame.
SACED therefore identifies compatible prior experts on a shared reference pool $\mathcal{D}_{\mathrm{ref}}$~\citep{zheng2023zscl,yu2024snd} and aligns the new expert with their attention-level reasoning through a collaborative distillation loss.
The frozen teacher is used only as a common reference anchor~\citep{li2018lwf,hinton2015kd,Dzagoruyko2017attn_attn}.
The goal is graceful failure: even when the router selects a nearby but non-specialized expert, the activated module remains close to the representation geometry of compatible experts.

\textbf{Expert Candidate Selection.}
SACED evaluates previously learned experts on the shared reference pool $\mathcal{D}_{\mathrm{ref}}$, which may be out-of-distribution with respect to the current task, without replaying old task data.
A prior expert is rated by a weighted compatibility score combining three signals: (i) co-activation overlap on the top-$K_g$ gating assignments~\citep{shazeer2017smoe} with weight $0.5$, (ii) top-$1$ class-prediction agreement~\citep{stanton2021kdwork} with weight $0.3$, and (iii) cosine similarity of class-probability distributions~\citep{chen2025mergingsmoe} with weight $0.2$.
The top-$k_{\mathrm{ref}}$ scoring experts above a threshold form $\mathcal{S}_k$, the selected expert set for task $k$.
This keeps cross-expert transfer local to experts likely to share useful representation geometry. If no prior expert passes the threshold, $\mathcal{S}_k=\emptyset$ and the $\mathcal{L}_{SACED}$ is set to zero for the current task.

\textbf{Teacher-Anchored Attention Alignment.}
SACED aligns the current expert with selected prior experts using attention-based supervision~\citep{zagoruyko2017attention_transfer}, since attention maps provide a label-agnostic signal of representation use~\citep{jia2021align,li2021align_fuse}.
To make attention comparisons well-defined, the student expert, the selected prior expert, and the teacher are all evaluated on the same reference input $x_{\mathrm{ref}}\in\mathcal{D}_{\mathrm{ref}}$.
This avoids comparing attention maps produced from different images and uses the teacher as a shared reference frame.
Given two attention maps $A$ and $B$, SACED first applies a fixed pooling operation after softmax normalization and then computes a cosine distance:
\begin{equation}
\small
\begin{aligned}
d_{\mathrm{attn}}(A,B)
=
1 -
\cos\!\left(\rho(A),\rho(B)\right), \ \ \text{ where } \rho(A)
=
\operatorname{pool}\!\left(\operatorname{softmax}(A)\right), 
\end{aligned}
\label{eq:saced_attn_distance}
\end{equation}
where $\operatorname{pool}(\cdot)$ denotes the same pooling operation over attention dimensions for all experts, such as averaging over heads, layers, or tokens.
The SACED loss then matches the student--teacher attention distance to the selected-prior--teacher attention distance on the same reference input:
\begin{equation}
\small
\begin{aligned}
\mathcal{L}_{\mathrm{SACED}}
=
\mathbb{E}_{\substack{
x_{\mathrm{ref}}\sim\mathcal{D}_{\mathrm{ref}} \\
j\sim\mathrm{Unif}(\mathcal{S}_k)
}}
\left[
\left|
d_{\mathrm{attn}}\!\left(
A_s(x_{\mathrm{ref}}),
A_t(x_{\mathrm{ref}})
\right)
-
\operatorname{sg}\!\left(
d_{\mathrm{attn}}\!\left(
A_j(x_{\mathrm{ref}}),
A_t(x_{\mathrm{ref}})
\right)
\right)
\right|
\right],
\end{aligned}
\label{eq:saced_loss}
\end{equation}
where $A_s$, $A_j$, and $A_t$ denote the attention maps of the student expert, the $j$-th selected prior expert, and the frozen teacher, respectively.
The operator $\operatorname{sg}(\cdot)$ denotes stop-gradient, since the selected prior experts and the teacher are fixed during the current task.
If $\mathcal{S}_k=\emptyset$, we define $\mathcal{L}_{\mathrm{SACED}}=0$.

This objective transfers teacher-anchored attention geometry rather than class logits.
The student is therefore encouraged to occupy a representation frame compatible with selected prior experts while remaining free to specialize its output predictions for the current task.
As a result, SACED reduces the damage caused by inevitable routing mistakes without directly modifying the router.

\subsection{Sub-Linear Expert Growth}
\label{sec:theory}

COMET's expert pool is not intended to grow once every task; its size is controlled by how often the stream presents distributions that are genuinely novel in a fixed teacher-induced representation.
Let $\delta_k=\min_{m\le M_{k-1}}d(\psi(\mathcal{D}_k),\psi(E_m))$ be the distance from task $k$ to the current expert pool, where $d$ is a metric in the frozen-teacher feature space and $\delta_1=+\infty$ because the pool starts empty, and call task $k$ $\epsilon$-novel when $\delta_k>\epsilon$.
For any merge-or-keep strategy that adds at most one expert per task and keeps a new expert only on $\epsilon$-novel tasks, the expert count satisfies
\begin{equation}
\label{eq:novelty_growth_bound}
\mathbb{E}[M_K]
\le
\sum_{k=1}^{K}\Pr(\delta_k>\epsilon).
\end{equation}
Hence, if the cumulative novelty is sub-linear, $\sum_{k=1}^{K}\Pr(\delta_k>\epsilon)=o(K)$, then $\mathbb{E}[M_K]=o(K)$.
This is a sufficient-condition guarantee rather than a worst-case claim: it covers streams with revisits, finite-cover support, or sublinear cumulative drift, while excluding adversarial streams that continually move into new regions.
Appendix~\ref{app:growth_proof} gives the full proof and extends the bound to practical policies with occasional false expansions, including the margin condition needed to translate bandit regret into a sub-linear false-expansion rate.
We use the bound to state when expert growth stays sub-linear and rely on experiments to check CMAB against it.

\subsection{Expert Merging via Contextual Multi-Armed Bandit}
\label{sec:cmab}
\label{sec:bandit}

TAD and SACED make experts compatible, but retaining every new expert still increases storage with the stream.
COMET therefore uses a contextual multi-armed bandit (CMAB) as a post-training capacity controller.
After task $k$, the new expert $E_{\mathrm{new}}$ is either kept or merged into a previous expert, i.e.,
$a_k\in\{\textsc{keep}(E_{\mathrm{new}})\}\cup\{\textsc{merge}(E_{\mathrm{new}},E_c):E_c\in\mathcal{E}_{<k}\}$.
Because this decision is made after expert optimization, capacity control is decoupled from the representation losses.

COMET first forms a small candidate set $\mathcal{C}_k$ using lightweight compatibility signals such as centroid similarity, routing overlap, and validation-based cross-performance.
Classification accuracy and rewards are computed on the labeled held-out split $\mathcal{V}_k$ of the current task; the unlabeled reference pool $\mathcal{D}_{\mathrm{ref}}$ is used only for label-free feature, routing, and distribution-agreement signals.
For candidate $E_c$, the context is $\phi_{k,c}=[\mathrm{Acc}_k,m_{k,c},-d_{k,c},\kappa_k,-\ell_c]^\top$, where $m_{k,c}$ measures merge compatibility, $d_{k,c}$ estimates degradation risk, $\kappa_k$ is capacity pressure, and $\ell_c$ penalizes repeatedly merging into an overloaded expert.
The bandit score is
\begin{equation}
B_k(c)
=
w^\top \phi_{k,c}
+
\frac{\beta_k}{1+N_k},
\label{eq:cmab_score}
\end{equation}
with learned weights $w$, decaying exploration bonus $\beta_k$, and update count $N_k$.
Let $c^\star=\arg\max_{c\in\mathcal{C}_k}B_k(c)$.
COMET merges into $E_{c^\star}$ when $B_k(c^\star)>\tau_k$ and otherwise keeps the new expert.
Thus, merging is an outcome-adaptive decision over compatibility, risk, capacity pressure, and load, rather than a fixed similarity threshold.

After the selected action is evaluated, both \textsc{merge} and \textsc{keep} outcomes update the policy.
The reward is the validation accuracy of the temporary merged expert minus a small merge cost for \textsc{merge}, and the validation accuracy of $E_{\mathrm{new}}$ for \textsc{keep}.
To avoid monotone drift from positive-valued accuracy rewards, CMAB uses a centered advantage $A_k=r_k-\bar r_k$ and an action sign $y_k=+1$ for \textsc{merge}, $y_k=-1$ for \textsc{keep}; the update follows the form $w\leftarrow\Pi(w+\eta_w A_k y_k\phi_{k,c^\star})$ with EMA smoothing.
The threshold $\tau_k$ is adapted by the same advantage signal, lowering after successful merges and increasing after failed merges or successful keeps.
Full candidate construction, reward shaping, threshold update, and hyperparameters are given in Appendix~\ref{app:cmab_impl}.

\subsection{Inference Routing with Reconstruction}
\label{sec:routing}

COMET routes without task identity and without teacher predictions; the frozen teacher image encoder supplies only the feature space.
For each retained expert $E_m$, we train a lightweight autoencoder router $R_m$ on that expert's task data in the frozen CLIP image-feature space.
At inference, each router produces a reconstruction error
$r_m(x)=d^{-1}\|R_m(g(x))-g(x)\|_2^2$, where $g(\cdot)$ is the frozen image feature extractor.
To make errors comparable across experts, COMET converts them into calibrated likelihood-ratio scores
\begin{equation}
\ell_m(x)
=
\log p(\log r_m(x)\mid m)
-
\log p(\log r_m(x)\mid \mathrm{bg}),
\end{equation}
where the foreground distribution is estimated from expert-specific training data and the background distribution is estimated from the shared reference pool.
The scores are converted into gates by 
$w_m(x)=\mathrm{softmax}(\ell_m(x)/\tau_g)_m$, and only the top-$k$ experts, with $k\le2$, are evaluated.
Their class-text similarity logits are combined by gate-weighted averaging.
Calibration details and the top-$k$ activation rule are given in Appendix~\ref{app:routing_impl}.
This preserves COMET's role separation: TAD and SACED shape expert geometry, CMAB controls retained capacity, and the reconstruction router selects experts at test time without semantic task cues.

\section{Experiments}

\subsection{Experimental Setup}

\textbf{Benchmarks.}
We evaluate COMET on two complementary continual-learning benchmarks.
\emph{X-TAIL}~\citep{xu2024x-tail} sequences ten image-classification datasets (Aircraft~\citep{maji2013aircraft}, Caltech101~\citep{li2004caltech101}, DTD~\citep{cimpoi2014dtd}, EuroSAT~\citep{helber2019eurosat}, Flowers~\citep{nilsback2008flowers}, Food~\citep{bossard2014food101}, MNIST~\citep{deng2012mnist}, OxfordPet~\citep{parkhi2012oxfordpets}, StanfordCars~\citep{krause2013standfordcars}, and SUN397~\citep{xiao2010sun397}). Each treated as an individual downstream task following~\citep{zheng2023zscl,luo2025lada}.
A second \emph{heterogeneous 4-task} sequence covers medical imaging (Covid19~\citep{samir2023covid}), microscopy (SEM~\citep{iwata2022sem}), and histopathology (Kather-Texture~\citep{kather2016kather_texture}, NCT100K~\citep{kather2018nct100k}); these domains lie far from CLIP pre-training and probe whether COMET continues to adapt under large distribution shift while retaining prior knowledge.
We report full-shot results in the main text; 16-shot results are in Appendix~\ref{app:main_16}.

\textbf{Evaluation Metrics.}
We use standard continual learning metrics. \emph{Last} is the final accuracy after the full stream and reflects end-of-stream retention. \emph{Transfer} averages accuracy on unseen tasks across the stream and captures forward transfer. \emph{Average} is the mean accuracy across all evaluation steps.

\textbf{Implementation Details.}
The teacher is the frozen CLIP ViT-B/16~\citep{radford2021clip}, and each per-task expert is a TinyCLIP~\citep{wu2023tinyclip} with $\sim 58$M parameters.
For each task, we train one expert for $1{,}000$ iterations with AdamW~\citep{loshchilov2019decoupled} at a learning rate of $10^{-5}$ and a batch size of $64$.
SACED draws its reference distribution from a fixed pool of $10{,}800$ image--text pairs from the Conceptual Captions validation set~\citep{sharma2018conceptual}, which also supplies the contextual-bandit reward.
Default loss weights are $\lambda_1{=}0.2$ for the cross-modal TAD term and $\lambda_2{=}1.0$ for SACED; at inference, the autoencoder router activates the top-$k$ experts with $k{\le}2$.
The full training schedule, CMAB hyperparameters, and routing rule are specified in Appendix~\ref{app:impl}.

\subsection{Main Results}

\paragraph{X-TAIL.}
We compare COMET against fine-tuning methods (LwF, WiSE-FT, ZSCL~\citep{zheng2023zscl}) and parameter-efficient methods (MoE-Adapters, LADA~\citep{luo2025lada}); per-task and aggregate metrics are reported in \Cref{tab:main_results}.
COMET achieves the best aggregate Transfer ($62.6$), Average ($76.1$), and Last ($88.1$), improving over LADA by $1.5$, $1.4$, and $1.4$ percentage points (pp), respectively.
The largest Last improvement over LADA occurs on Aircraft ($+19.5$\,pp; $72.4$ vs.\ $52.9$).
These gains are obtained with a $2.6{\times}$ smaller per-expert backbone (TinyCLIP $58$M vs.\ CLIP-B/16 $150$M).

\begin{table}[t]
\caption{Comparison of different continual learning methods on X-TAIL full-shot for each task in terms of \textit{Transfer}, \textit{Average}, and \textit{Last} scores (\%). The best results are highlighted with \textbf{bold} style.}
\label{tab:main_results}
\centering
\resizebox{\textwidth}{!}{%
\begin{tabular}{lccccccccccc}
\toprule
\textbf{Method} & \rotatebox{45}{Aircraft} & \rotatebox{45}{Caltech101} & \rotatebox{45}{DTD} & \rotatebox{45}{EuroSAT} & \rotatebox{45}{Flowers} & \rotatebox{45}{Food} & \rotatebox{45}{MNIST} & \rotatebox{45}{Pets} & \rotatebox{45}{Cars} & \rotatebox{45}{Sun397} & Average \\
\midrule
\multicolumn{12}{l}{\texttt{TinyCLIP ViT-39M/16@224px}} \\
Zero-shot & 17.0 & 75.6 & 39.0 & 40.6 & 65.0 & 77.4 & 22.0 & 78.9 & 47.8 & 67.2 & 53.1 \\
\multicolumn{12}{l}{\texttt{CLIP ViT-B/16@224px}} \\
Zero-shot & 23.8 & 74.3 & 36.4 & 37.4 & 64.1 & 83.4 & 43.9 & 87.8 & 65.5 & 60.8 & 57.7 \\
\midrule
\textbf{Transfer} \\
LwF~\citep{li2018lwf} & -- & 64.8 & 29.1 & 7.6 & 49.7 & 72.0 & 33.3 & 64.1 & 29.8 & 53.7 & 44.9 \\
WiSE-FT~\citep{wortsman2022robust-ft} & -- & 68.3 & 30.8 & 17.4 & 59.2 & 82.5 & 41.8 & 74.9 & 48.3 & 59.7 & 53.7 \\
ZSCL~\citep{zheng2023zscl} & -- & 71.7 & 32.6 & 15.2 & 63.0 & 85.0 & 29.6 & 83.8 & 59.3 & 63.2 & 55.9 \\
MoE-Adapters~\citep{yu2024moe-adapter} & -- & 76.2 & 30.9 & 19.2 & 67.7 & \textbf{87.8} & \textbf{50.2} & \textbf{88.4} & 63.0 & 59.8 & 60.3 \\
LADA~\citep{luo2025lada} & -- & 74.9 & 35.5 & 36.5 & 66.1 & 83.9 & 40.3 & 88.0 & \textbf{65.3} & 59.6 & 61.1 \\ \hdashline
COMET (Ours) & --   & \textbf{89.3} & \textbf{47.1} & \textbf{49.9} & \textbf{67.8} & 77.5 & 32.9 & 79.1 & 51.3 & \textbf{68.7} & \textbf{62.6} \\
\midrule
\textbf{Average} \\
LwF & 18.3 & 80.5 & 42.3 & 35.1 & 49.1 & 80.9 & 57.5 & 71.1 & 37.0 & 56.3 & 52.8 \\
WiSE-FT & 28.1 & 80.0 & 43.8 & 38.9 & 62.9 & 86.6 & 61.9 & 79.5 & 54.9 & 61.6 & 59.8 \\
ZSCL & 40.3 & 78.2 & 50.0 & 40.6 & 73.2 & \textbf{87.6} & 56.2 & 86.3 & 64.5 & 64.8 & 64.2 \\
MoE-Adapters & 50.8 & 78.7 & 33.8 & 47.0 & 69.0 & 87.5 & 56.7 & 88.3 & 62.2 & 60.2 & 63.4 \\
Primal-RAIL~\citep{xu2024x-tail} & 48.7 & 89.7 & 58.7 & 74.2 & 83.8 & 86.4 & \textbf{67.3} & 88.7 & 69.3 & 64.9 & 73.2 \\
LADA & 52.8 & 93.7 & 66.8 & 77.7 & 85.1 & 86.3 & 63.4 & \textbf{89.8} & \textbf{69.6} & 61.7 & 74.7 \\ \hdashline
COMET (Ours) & \textbf{72.4} & \textbf{95.1} & \textbf{71.9} & \textbf{83.8} & \textbf{86.3} & 80.9 & 59.1 & 83.1 & 58.6 & \textbf{69.9} & \textbf{76.1} \\
\midrule
\textbf{Last} \\
LwF & 9.8 & 75.2 & 48.4 & 24.7 & 50.9 & 89.4 & 90.4 & 86.8 & 59.5 & 79.3 & 61.4 \\
WiSE-FT & 31.0 & 78.3 & 47.5 & 33.0 & 71.0 & 90.5 & 88.4 & 88.1 & 78.9 & 78.8 & 68.5 \\
ZSCL & 35.4 & 79.2 & 52.9 & 34.9 & 80.0 & \textbf{90.6} & 95.4 & 91.9 & 84.8 & 78.6 & 72.4 \\
MoE-Adapters & 50.8 & 79.0 & 34.6 & 58.9 & 69.9 & 87.1 & 66.4 & 88.1 & 59.1 & 64.1 & 65.8 \\
Primal-RAIL & 45.6 & 94.2 & 70.9 & 94.2 & 96.6 & 89.0 & 98.1 & 93.5 & 82.0 & 76.1 & 84.0 \\
LADA & 52.9 & \textbf{96.4} & 76.0 & 95.2 & 97.9 & 88.8 & 98.0 & \textbf{94.2} & 86.6 & \textbf{80.7} & 86.7 \\ \hdashline
COMET (Ours) & \textbf{72.4} & 91.5 & \textbf{77.7} & \textbf{98.1} & \textbf{98.6} & 84.3 & \textbf{98.6} & 92.2 & \textbf{87.7} & 80.2 & \textbf{88.1} \\
\bottomrule
\end{tabular}%
}
\end{table}

\paragraph{Heterogeneous 4-task sequence.}
\Cref{tab:heterogeneous_results} reports the same metrics on the medical, microscopy, and histopathology sequence, where pre-trained CLIP priors no longer dominate.
COMET improves Transfer by $+18.7$\,pp over ZSCL ($31.2$ vs.\ $12.5$), Average by $+20.2$\,pp ($72.9$ vs.\ $52.7$), and Last by $+3.5$\,pp ($96.7$ vs.\ $93.2$).
On NCT100K alone, COMET reaches $40.5$ Transfer, whereas every baseline remains below $1.6$.
These gains show that full per-task experts provide the capacity that lightweight adapters on a frozen CLIP lack, while TAD anchors adaptation to a shared teacher geometry.

\subsection{Ablation Study}

We isolate the two distillation losses and the contextual-bandit merging policy on the X-TAIL 10-dataset protocol.
\Cref{tab:compact_results} (a) shows that TAD and SACED are complementary: each improves Transfer, Average, and Last over the base model, and using both gives the strongest results in both full-shot and 16-shot settings.
Sub-component decompositions of TAD and SACED, hyperparameter and reference-data sensitivity, and additional analyses are reported in Appendices~\ref{app:sensitivity}--\ref{app:saced_select}.

\begin{table}[t]
\centering
\caption{\textbf{Heterogeneous 4-dataset full-shot results.}
Transfer, Average, and Last scores (\%).
The best result in each metric and column is in \textbf{bold}.}
\label{tab:heterogeneous_results}

\begingroup
\small
\setlength{\tabcolsep}{5.5pt}
\small
\renewcommand{\arraystretch}{1.04}
\setlength{\aboverulesep}{0.45ex}
\setlength{\belowrulesep}{0.45ex}

\begin{tabular}{@{}lrrrrr@{}}
\toprule
\textbf{Method} & \textbf{Cov19} & \textbf{KT} &
\textbf{NCT} & \textbf{SEM} & \textbf{Avg.} \\
\midrule

\multicolumn{6}{@{}l}{\textit{Zero-shot}} \\
TinyCLIP-ZS & 48.6 & 5.8  & 17.0 & 22.3 & 23.4 \\
CLIP-B-ZS   & 48.9 & 16.3 & 15.6 & 24.9 & 26.4 \\
\midrule

\multicolumn{6}{@{}l}{\textit{Transfer}} \\
CLIP-FT      & -- & 10.3 & 0.0 & 21.5 & 10.6 \\
WISE-FT      & -- & 12.0 & 0.1 & 23.6 & 11.9 \\
ZSCL         & -- & 12.7 & 0.1 & 24.8 & 12.5 \\
MoE-Adapters & -- & 10.7 & 1.6 & 10.8 & 7.7 \\
LADA         & -- & 0.0  & 0.0 & 22.6 & 7.5 \\
\addlinespace[2pt]
\textbf{COMET (Ours)} & -- &
\textbf{19.8} & \textbf{40.5} &
\textbf{33.4} & \textbf{31.2} \\
\midrule

\multicolumn{6}{@{}l}{\textit{Average}} \\
CLIP-FT      & 29.4 & 30.8 & 46.7 & 40.4 & 36.9 \\
WISE-FT      & 31.1 & 38.8 & 44.3 & 41.9 & 39.0 \\
ZSCL         & 30.8 & 45.1 & 41.6 & 41.5 & 39.8 \\
MoE-Adapters & 95.1 & 75.2 & 46.5 & 32.0 & 51.2 \\
LADA         & 72.6 & 73.5 & 45.7 & 38.9 & 52.7 \\
\addlinespace[2pt]
\textbf{COMET (Ours)} &
\textbf{95.5} & \textbf{77.3} &
\textbf{69.7} & \textbf{49.0} & \textbf{72.9} \\
\midrule

\multicolumn{6}{@{}l}{\textit{Last}} \\
CLIP-FT      & 27.4 & 5.4  & 87.8 & 96.5 & 54.3 \\
WISE-FT      & 29.1 & 27.5 & 78.3 & \textbf{96.8} & 57.9 \\
ZSCL         & 26.3 & 43.1 & 82.1 & 91.5 & 60.8 \\
MoE-Adapters & 94.8 & 96.5 & 91.1 & 90.2 & 93.2 \\
LADA         & 72.7 & \textbf{97.0} & 91.4 & 88.0 & 87.3 \\
\addlinespace[2pt]
\textbf{COMET (Ours)} &
\textbf{95.5} & 96.5 &
\textbf{99.0} & 95.8 & \textbf{96.7} \\
\bottomrule
\end{tabular}

\par\vspace{3pt}
{\footnotesize\raggedright
ZS: zero-shot; Cov19: Covid19; KT: Kather-Texture; NCT: NCT 100K. \par}

\endgroup
\end{table}

\begin{table}[t]
\centering
\caption{\textbf{Ablations and inference cost.}
(a) TAD/SACED ablation on X-TAIL; (b) merge policy;
(c) inference cost. Best results are in \textbf{bold}.}
\label{tab:compact_results}

\begingroup
\footnotesize
\setlength{\tabcolsep}{1.5pt}
\renewcommand{\arraystretch}{1.08}
\setlength{\aboverulesep}{0.4ex}
\setlength{\belowrulesep}{0.4ex}
\setlength{\heavyrulewidth}{0.6pt}
\setlength{\lightrulewidth}{0.35pt}

\noindent
\begin{minipage}[t]{0.435\linewidth}
\vspace{0pt}
\centering

\textbf{(a) TAD/SACED ablation}\par
\vspace{3pt}

\begin{tabular*}{\linewidth}{@{\extracolsep{\fill}}ccrrr@{}}
\toprule
TAD & SACED & Trans. & Avg. & Last \\
\midrule

\multicolumn{5}{@{}l}{\textit{Full-shot}} \\
-- & -- & 56.6 & 73.9 & 85.6 \\
$\checkmark$ & -- & 61.3 & 75.4 & 87.6 \\
-- & $\checkmark$ & 59.9 & 75.1 & 85.3 \\
$\checkmark$ & $\checkmark$ &
\textbf{62.6} & \textbf{76.1} & \textbf{88.1} \\
\midrule

\multicolumn{5}{@{}l}{\textit{16-shot}} \\
-- & -- & 52.7 & 65.1 & 73.3 \\
$\checkmark$ & -- & 57.7 & 69.2 & 77.9 \\
-- & $\checkmark$ & 53.3 & 66.3 & 73.8 \\
$\checkmark$ & $\checkmark$ &
\textbf{60.9} & \textbf{72.6} & \textbf{83.3} \\
\bottomrule
\end{tabular*}

\end{minipage}\hfill%
\begin{minipage}[t]{0.54\linewidth}
\vspace{0pt}
\centering

\textbf{(b) Merge policy}\par
\vspace{3pt}

\begin{tabular*}{\linewidth}{@{\extracolsep{\fill}}lrrrrr@{}}
\toprule
Variant & $M$ & Stored & Trans. & Avg. & Last \\
\midrule
Random & 3  & 261 & 58.9 & 64.9 & 69.1 \\
Merge  & 1  & 145 & 58.4 & 58.4 & 58.4 \\
Expand & 10 & 667 & 61.6 & 75.8 & 87.3 \\
COMET  & 8  & 551 &
\textbf{62.6} & \textbf{76.1} & \textbf{88.1} \\
\bottomrule
\end{tabular*}

\vspace{6pt}

\textbf{(c) Inference cost}\par
\vspace{3pt}

\begin{tabular*}{\linewidth}{@{\extracolsep{\fill}}lrrrr@{}}
\toprule
Method & Active & GFLOPs & Lat. & Img/s \\
\midrule
LwF/WISE-FT/ZSCL & $\sim150$ & 17.6 & 166.3 & 385 \\
MoE-Adapt.       & $\sim180$ & 18.4 & 196.5 & 326 \\
LADA            & 96--150   & \textbf{8.6} & 176.9 & 362 \\
COMET           & $\sim145^{\dagger}$ &
12.2 & \textbf{107.1} & \textbf{598} \\
\bottomrule
\end{tabular*}

\end{minipage}

\par\vspace{4pt}
{\scriptsize\raggedright
$M$: retained experts; Stored/Active: million parameters;
Lat.: milliseconds.
Expand/Merge: Always-Expand/Always-Merge; $^{\dagger}$Top-$1$ path.}

\endgroup
\end{table}

\paragraph{TAD and SACED.}
TAD alone improves Transfer by $+4.7$\,pp ($56.6\!\to\!61.3$), Average by $+1.5$\,pp, and Last by $+2.0$\,pp over the no-distillation variant.
SACED alone improves Transfer and Average by $+3.3$ and $+1.2$\,pp, respectively, but slightly reduces Last by $0.3$\,pp.
Combining TAD and SACED yields the strongest results in this ablation: $62.6$ Transfer, $76.1$ Average, and $88.1$ Last.
In particular, adding SACED to TAD improves the three metrics by $+1.3$, $+0.7$, and $+0.5$\,pp, respectively, supporting the benefit of their combination rather than uniform improvements in isolation.

\paragraph{Expert merging strategy.}
Holding the teacher, expert architecture, TAD, SACED, router, and merge operator fixed, we vary only the merge-or-keep decision rule.
\Cref{tab:compact_results} (b) shows that CMAB retains $8$ experts, compared with $10$ for Always-Expand, while improving Transfer, Average, and Last by $+1.0$, $+0.3$, and $+0.8$\,pp, respectively.
Always-Expand stores more parameters but underperforms CMAB, indicating that simply adding experts does not solve routing and capacity allocation.
Always-Merge and random merging degrade sharply, confirming that structured outcome-based merge decisions are necessary.
The full seven-variant comparison and cross-sequence robustness analysis are reported in Appendix~\ref{app:cmab_extended}.

\paragraph{Routing and robustness.}
Across the ten task-specific routers, the autoencoder router identifies the correct task for $95.9\%$ of test inputs at top-$1$ and $99.4\%$ at top-$2$.
After merging, the router sends $2.6\%$ of test samples ($2{,}069$ of $80{,}939$) away from their oracle expert, and SACED turns these errors into graceful degradation: $52.5\%$ of misrouted samples are still correctly classified, and pairwise cross-expert representation similarity rises by CKA $+0.009$ relative to a no-SACED control (Appendix~\ref{app:routing}).
Across four task orderings on X-TAIL, Average is $72.8 \pm 2.2$; per-task forgetting, measured on the three re-ordered runs, is $2.2 \pm 0.1$\,pp.
The four key hyperparameters ($\lambda_1$, $\lambda_2$, the candidate gate $\theta_{\mathrm{sim}}$, and the initial merge threshold $\tau_1$) are robust over an order of magnitude, and SACED requires only a few thousand diverse image--text reference pairs.
Per-ordering results, hyperparameter sweeps, and reference-data ablations are reported in Appendices~\ref{app:order}, \ref{app:sensitivity}, and \ref{app:refdata}, respectively.

\subsection{Analysis of Computational Cost}

\Cref{tab:compact_results} (c) reports that COMET achieves the lowest latency and highest throughput because inference uses the teacher vision encoder, the lightweight autoencoder router, and only the selected TinyCLIP expert; the teacher text encoder is not executed.
Although the retained pool after X-TAIL stores roughly $551$M parameters, the active top-$1$ path uses only $\sim145$M parameters, and the top-$2$ path remains below $203$M.
Training adds $\sim2.7$\,GB peak VRAM because only the current expert receives gradients.
A detailed training and storage breakdown is in Appendix~\ref{app:cost_extended}.
LwF, WiSE-FT, and ZSCL have identical inference cost because all deploy a single full-CLIP backbone; their differences are only in training objectives.

\section{Conclusion}
\label{sec:conclusion}
A central question for continual VLMs is how to keep adapting when tasks arrive without replay, task identity, or full-backbone retraining.
We argue that this requires separating the roles that prior systems often couple.
COMET is a step toward this direction: it uses a frozen teacher only as a multimodal geometry reference, TAD to preserve tri-affinity structure, SACED to reduce routing-induced expert mismatch, CMAB to regulate capacity after training, and reconstruction routing to select sparse experts without teacher semantics.
Across X-TAIL and a heterogeneous medical, microscopy, and histopathology stream, this separation improves Average and Last accuracy, matches the best Transfer score, and keeps the retained expert pool compact.

\textbf{Limitations.} COMET's sub-linear storage behavior depends on diminishing cumulative novelty, such as revisits or bounded representation support.
Longer streams that keep entering new domains may still grow the expert pool linearly, and we prove no regret guarantee for CMAB.
Explicit storage budgets and capacity control with provable guarantees remain directions for future work.

\newpage
\small
\bibliographystyle{plainnat}
\bibliography{neurips_references}

@inproceedings{zheng2023zscl,
  author       = {Zangwei Zheng and
                  Mingyuan Ma and
                  Kai Wang and
                  Ziheng Qin and
                  Xiangyu Yue and
                  Yang You},
  title        = {Preventing Zero-Shot Transfer Degradation in Continual Learning of
                  Vision-Language Models},
  booktitle    = {{IEEE/CVF} International Conference on Computer Vision, {ICCV} 2023,
                  Paris, France, October 1-6, 2023},
  pages        = {19068--19079},
  publisher    = {{IEEE}},
  year         = {2023},
}

@inproceedings{luo2025lada,
  author       = {Mao{-}Lin Luo and
                  Zi{-}Hao Zhou and
                  Tong Wei and
                  Min{-}Ling Zhang},
  title        = {{LADA:} Scalable Label-Specific {CLIP} Adapter for Continual Learning},
  booktitle    = {Forty-second International Conference on Machine Learning, {ICML}
                  2025, Vancouver, BC, Canada, July 13-19, 2025},
  publisher    = {OpenReview.net},
  year         = {2025},
}

@inproceedings{xu2024x-tail,
  author       = {Yicheng Xu and
                  Yuxin Chen and
                  Jiahao Nie and
                  Yusong Wang and
                  Huiping Zhuang and
                  Manabu Okumura},
  editor       = {Amir Globersons and
                  Lester Mackey and
                  Danielle Belgrave and
                  Angela Fan and
                  Ulrich Paquet and
                  Jakub M. Tomczak and
                  Cheng Zhang},
  title        = {Advancing Cross-domain Discriminability in Continual Learning of Vision-Language Models},
  booktitle    = {Advances in Neural Information Processing Systems 38: Annual Conference
                  on Neural Information Processing Systems 2024, NeurIPS 2024, Vancouver,
                  BC, Canada, December 10 - 15, 2024},
  year         = {2024},
}

@inproceedings{yu2024moe-adapter,
  author       = {Jiazuo Yu and
                  Yunzhi Zhuge and
                  Lu Zhang and
                  Ping Hu and
                  Dong Wang and
                  Huchuan Lu and
                  You He},
  title        = {Boosting Continual Learning of Vision-Language Models via Mixture-of-Experts
                  Adapters},
  booktitle    = {{IEEE/CVF} Conference on Computer Vision and Pattern Recognition,
                  {CVPR} 2024, Seattle, WA, USA, June 16-22, 2024},
  pages        = {23219--23230},
  publisher    = {{IEEE}},
  year         = {2024},
}

@article{maji2013aircraft,
  author       = {Subhransu Maji and
                  Esa Rahtu and
                  Juho Kannala and
                  Matthew B. Blaschko and
                  Andrea Vedaldi},
  title        = {Fine-Grained Visual Classification of Aircraft},
  journal      = {CoRR},
  volume       = {abs/1306.5151},
  year         = {2013},
}

@inproceedings{bossard2014food101,
  author       = {Lukas Bossard and
                  Matthieu Guillaumin and
                  Luc Van Gool},
  editor       = {David J. Fleet and
                  Tom{\'{a}}s Pajdla and
                  Bernt Schiele and
                  Tinne Tuytelaars},
  title        = {Food-101 - Mining Discriminative Components with Random Forests},
  booktitle    = {Computer Vision - {ECCV} 2014 - 13th European Conference, Zurich,
                  Switzerland, September 6-12, 2014, Proceedings, Part {VI}},
  series       = {Lecture Notes in Computer Science},
  volume       = {8694},
  pages        = {446--461},
  publisher    = {Springer},
  year         = {2014},
}

@article{deng2012mnist,
  author       = {Li Deng},
  title        = {The {MNIST} Database of Handwritten Digit Images for Machine Learning
                  Research [Best of the Web]},
  journal      = {{IEEE} Signal Process. Mag.},
  volume       = {29},
  number       = {6},
  pages        = {141--142},
  year         = {2012},
}

@article{helber2019eurosat,
  author       = {Patrick Helber and
                  Benjamin Bischke and
                  Andreas Dengel and
                  Damian Borth},
  title        = {EuroSAT: {A} Novel Dataset and Deep Learning Benchmark for Land Use
                  and Land Cover Classification},
  journal      = {{IEEE} J. Sel. Top. Appl. Earth Obs. Remote. Sens.},
  volume       = {12},
  number       = {7},
  pages        = {2217--2226},
  year         = {2019},
}

@article{kirkpatrick2016forgetting,
  author       = {James Kirkpatrick and
                  Razvan Pascanu and
                  Neil C. Rabinowitz and
                  Joel Veness and
                  Guillaume Desjardins and
                  Andrei A. Rusu and
                  Kieran Milan and
                  John Quan and
                  Tiago Ramalho and
                  Agnieszka Grabska{-}Barwinska and
                  Demis Hassabis and
                  Claudia Clopath and
                  Dharshan Kumaran and
                  Raia Hadsell},
  title        = {Overcoming catastrophic forgetting in neural networks},
  journal      = {CoRR},
  volume       = {abs/1612.00796},
  year         = {2016},
}

@article{li2018lwf,
  author       = {Zhizhong Li and
                  Derek Hoiem},
  title        = {Learning without Forgetting},
  journal      = {{IEEE} Trans. Pattern Anal. Mach. Intell.},
  volume       = {40},
  number       = {12},
  pages        = {2935--2947},
  year         = {2018},
}

@inproceedings{nilsback2008flowers,
  author       = {Maria{-}Elena Nilsback and
                  Andrew Zisserman},
  title        = {Automated Flower Classification over a Large Number of Classes},
  booktitle    = {Sixth Indian Conference on Computer Vision, Graphics {\&} Image
                  Processing, {ICVGIP} 2008, Bhubaneswar, India, 16-19 December 2008},
  pages        = {722--729},
  publisher    = {{IEEE} Computer Society},
  year         = {2008},
}

@inproceedings{rebuffi2017icarl,
  author       = {Sylvestre{-}Alvise Rebuffi and
                  Alexander Kolesnikov and
                  Georg Sperl and
                  Christoph H. Lampert},
  title        = {iCaRL: Incremental Classifier and Representation Learning},
  booktitle    = {2017 {IEEE} Conference on Computer Vision and Pattern Recognition,
                  {CVPR} 2017, Honolulu, HI, USA, July 21-26, 2017},
  pages        = {5533--5542},
  publisher    = {{IEEE} Computer Society},
  year         = {2017},
}

@inproceedings{wortsman2022robust-ft,
  author       = {Mitchell Wortsman and
                  Gabriel Ilharco and
                  Jong Wook Kim and
                  Mike Li and
                  Simon Kornblith and
                  Rebecca Roelofs and
                  Raphael Gontijo Lopes and
                  Hannaneh Hajishirzi and
                  Ali Farhadi and
                  Hongseok Namkoong and
                  Ludwig Schmidt},
  title        = {Robust fine-tuning of zero-shot models},
  booktitle    = {{IEEE/CVF} Conference on Computer Vision and Pattern Recognition,
                  {CVPR} 2022, New Orleans, LA, USA, June 18-24, 2022},
  pages        = {7949--7961},
  publisher    = {{IEEE}},
  year         = {2022},
}

@inproceedings{cimpoi2014dtd,
  author       = {Mircea Cimpoi and
                  Subhransu Maji and
                  Iasonas Kokkinos and
                  Sammy Mohamed and
                  Andrea Vedaldi},
  title        = {Describing Textures in the Wild},
  booktitle    = {2014 {IEEE} Conference on Computer Vision and Pattern Recognition,
                  {CVPR} 2014, Columbus, OH, USA, June 23-28, 2014},
  pages        = {3606--3613},
  publisher    = {{IEEE} Computer Society},
  year         = {2014},
}

@inproceedings{radford2021clip,
  author       = {Alec Radford and
                  Jong Wook Kim and
                  Chris Hallacy and
                  Aditya Ramesh and
                  Gabriel Goh and
                  Sandhini Agarwal and
                  Girish Sastry and
                  Amanda Askell and
                  Pamela Mishkin and
                  Jack Clark and
                  Gretchen Krueger and
                  Ilya Sutskever},
  editor       = {Marina Meila and
                  Tong Zhang},
  title        = {Learning Transferable Visual Models From Natural Language Supervision},
  booktitle    = {Proceedings of the 38th International Conference on Machine Learning,
                  {ICML} 2021, 18-24 July 2021, Virtual Event},
  series       = {Proceedings of Machine Learning Research},
  volume       = {139},
  pages        = {8748--8763},
  publisher    = {{PMLR}},
  year         = {2021},
}

@inproceedings{li2004caltech101,
  author       = {Li Fei{-}Fei and
                  Rob Fergus and
                  Pietro Perona},
  title        = {Learning Generative Visual Models from Few Training Examples: An Incremental
                  Bayesian Approach Tested on 101 Object Categories},
  booktitle    = {{IEEE} Conference on Computer Vision and Pattern Recognition Workshops,
                  {CVPR} Workshops 2004, Washington, DC, USA, June 27 - July 2, 2004},
  pages        = {178},
  publisher    = {{IEEE} Computer Society},
  year         = {2004},
}

@inproceedings{parkhi2012oxfordpets,
  author       = {Omkar M. Parkhi and
                  Andrea Vedaldi and
                  Andrew Zisserman and
                  C. V. Jawahar},
  title        = {Cats and dogs},
  booktitle    = {2012 {IEEE} Conference on Computer Vision and Pattern Recognition,
                  Providence, RI, USA, June 16-21, 2012},
  pages        = {3498--3505},
  publisher    = {{IEEE} Computer Society},
  year         = {2012},
}

@inproceedings{krause2013standfordcars,
  author       = {Jonathan Krause and
                  Michael Stark and
                  Jia Deng and
                  Li Fei{-}Fei},
  title        = {3D Object Representations for Fine-Grained Categorization},
  booktitle    = {2013 {IEEE} International Conference on Computer Vision Workshops,
                  {ICCV} Workshops 2013, Sydney, Australia, December 1-8, 2013},
  pages        = {554--561},
  publisher    = {{IEEE} Computer Society},
  year         = {2013},
}

@inproceedings{xiao2010sun397,
  author       = {Jianxiong Xiao and
                  James Hays and
                  Krista A. Ehinger and
                  Aude Oliva and
                  Antonio Torralba},
  title        = {{SUN} database: Large-scale scene recognition from abbey to zoo},
  booktitle    = {The Twenty-Third {IEEE} Conference on Computer Vision and Pattern
                  Recognition, {CVPR} 2010, San Francisco, CA, USA, 13-18 June 2010},
  pages        = {3485--3492},
  publisher    = {{IEEE} Computer Society},
  year         = {2010}, 
}

@inproceedings{wu2023tinyclip,
  author       = {Kan Wu and
                  Houwen Peng and
                  Zhenghong Zhou and
                  Bin Xiao and
                  Mengchen Liu and
                  Lu Yuan and
                  Hong Xuan and
                  Michael Valenzuela and
                  Xi Stephen Chen and
                  Xinggang Wang and
                  Hongyang Chao and
                  Han Hu},
  title        = {TinyCLIP: {CLIP} Distillation via Affinity Mimicking and Weight Inheritance},
  booktitle    = {{IEEE/CVF} International Conference on Computer Vision, {ICCV} 2023,
                  Paris, France, October 1-6, 2023},
  pages        = {21913--21923},
  publisher    = {{IEEE}},
  year         = {2023},
}

@inproceedings{sharma2018conceptual,
  title = {Conceptual Captions: A Cleaned, Hypernymed, Image Alt-text Dataset For Automatic Image Captioning},
  author = {Sharma, Piyush and Ding, Nan and Goodman, Sebastian and Soricut, Radu},
  booktitle = {Proceedings of ACL},
  year = {2018},
}

@article{samir2023covid,
  author       = {Benbakreti Samir and
                  Said Mwanahija and
                  Benbakreti Soumia and
                  Umut {\"{O}}zkaya},
  title        = {Deep Learning For Classification Of Chest X-Ray Images (Covid 19)},
  journal      = {CoRR},
  volume       = {abs/2301.02468},
  year         = {2023},
}

@misc{kather2018nct100k,
  author       = {Kather, Jakob Nikolas and Halama, Niels and Marx, Alexander},
  title        = {{100,000 histological images of human colorectal cancer and healthy tissue}},
  month        = apr,
  year         = 2018,
  publisher    = {Zenodo},
  version      = {v0.1},
}

@article{iwata2022sem,
  author  = {Hiroaki Iwata and Yoshihiro Hayashi and Aki Hasegawa and Kei Terayama and Yasushi Okuno},
  title   = {Classification of scanning electron microscope images of pharmaceutical excipients using deep convolutional neural networks with transfer learning},
  journal = {International Journal of Pharmaceutics: X},
  volume  = {4},
  pages   = {100135},
  year    = {2022},
  doi     = {10.1016/j.ijpx.2022.100135}
}

@article{kather2016kather_texture,
  title   = {Multi-class texture analysis in colorectal cancer histology},
  author  = {Kather, Jakob Nikolas and Weis, Cleo-Aron and Bianconi, Francesco and Melchers, Susanne M. and Schad, Lothar R. and Gaiser, Timo and Marx, Alexander and Z{\"o}llner, Frank Gerrit},
  journal = {Scientific Reports},
  year    = {2016},
  volume  = {6},
  pages   = {27988},
}

@inproceedings{shazeer2017smoe,
  author       = {Noam Shazeer and
                  Azalia Mirhoseini and
                  Krzysztof Maziarz and
                  Andy Davis and
                  Quoc V. Le and
                  Geoffrey E. Hinton and
                  Jeff Dean},
  title        = {Outrageously Large Neural Networks: The Sparsely-Gated Mixture-of-Experts
                  Layer},
  booktitle    = {5th International Conference on Learning Representations, {ICLR} 2017,
                  Toulon, France, April 24-26, 2017, Conference Track Proceedings},
  publisher    = {OpenReview.net},
  year         = {2017},
}

@inproceedings{stanton2021kdwork,
  author       = {Samuel Stanton and
                  Pavel Izmailov and
                  Polina Kirichenko and
                  Alexander A. Alemi and
                  Andrew Gordon Wilson},
  editor       = {Marc'Aurelio Ranzato and
                  Alina Beygelzimer and
                  Yann N. Dauphin and
                  Percy Liang and
                  Jennifer Wortman Vaughan},
  title        = {Does Knowledge Distillation Really Work?},
  booktitle    = {Advances in Neural Information Processing Systems 34: Annual Conference
                  on Neural Information Processing Systems 2021, NeurIPS 2021, December
                  6-14, 2021, virtual},
  pages        = {6906--6919},
  year         = {2021},
}

@inproceedings{chen2025mergingsmoe,
  author       = {I{-}Chun Chen and
                  Hsu{-}Shen Liu and
                  Wei{-}Fang Sun and
                  Chen{-}Hao Chao and
                  Yen{-}Chang Hsu and
                  Chun{-}Yi Lee},
  title        = {Retraining-free Merging of Sparse MoE via Hierarchical Clustering},
  booktitle    = {Forty-second International Conference on Machine Learning, {ICML}
                  2025, Vancouver, BC, Canada, July 13-19, 2025},
  publisher    = {OpenReview.net},
  year         = {2025},
}

@article{hadsell2020embracing,
  title={Embracing change: Continual learning in deep neural networks},
  author={Hadsell, Raia and Rao, Dushyant and Rusu, Andrei A and Pascanu, Razvan},
  journal={Trends in cognitive sciences},
  volume={24},
  number={12},
  pages={1028--1040},
  year={2020},
  publisher={Elsevier}
}

@article{de2021continual,
  title={A continual learning survey: Defying forgetting in classification tasks},
  author={De Lange, Matthias and Aljundi, Rahaf and Masana, Marc and Parisot, Sarah and Jia, Xu and Leonardis, Ale{\v{s}} and Slabaugh, Gregory and Tuytelaars, Tinne},
  journal={IEEE transactions on pattern analysis and machine intelligence},
  volume={44},
  number={7},
  pages={3366--3385},
  year={2021},
  publisher={IEEE}
}

@article{liu2023visual,
  title={Visual instruction tuning},
  author={Liu, Haotian and Li, Chunyuan and Wu, Qingyang and Lee, Yong Jae},
  journal={Advances in neural information processing systems},
  volume={36},
  pages={34892--34916},
  year={2023}
}

@inproceedings{liu2024improved,
  title={Improved baselines with visual instruction tuning},
  author={Liu, Haotian and Li, Chunyuan and Li, Yuheng and Lee, Yong Jae},
  booktitle={Proceedings of the IEEE/CVF conference on computer vision and pattern recognition},
  pages={26296--26306},
  year={2024}
}

@inproceedings{chen2024internvl,
  title={Internvl: Scaling up vision foundation models and aligning for generic visual-linguistic tasks},
  author={Chen, Zhe and Wu, Jiannan and Wang, Wenhai and Su, Weijie and Chen, Guo and Xing, Sen and Zhong, Muyan and Zhang, Qinglong and Zhu, Xizhou and Lu, Lewei and others},
  booktitle={Proceedings of the IEEE/CVF conference on computer vision and pattern recognition},
  pages={24185--24198},
  year={2024}
}

@article{wang2024qwen2,
  title={Qwen2-vl: Enhancing vision-language model's perception of the world at any resolution},
  author={Wang, Peng and Bai, Shuai and Tan, Sinan and Wang, Shijie and Fan, Zhihao and Bai, Jinze and Chen, Keqin and Liu, Xuejing and Wang, Jialin and Ge, Wenbin and others},
  journal={arXiv preprint arXiv:2409.12191},
  year={2024}
}

@inproceedings{radford2021learning,
  title={Learning transferable visual models from natural language supervision},
  author={Radford, Alec and Kim, Jong Wook and Hallacy, Chris and Ramesh, Aditya and Goh, Gabriel and Agarwal, Sandhini and Sastry, Girish and Askell, Amanda and Mishkin, Pamela and Clark, Jack and others},
  booktitle={International conference on machine learning},
  pages={8748--8763},
  year={2021},
  organization={PmLR}
}

@incollection{mccloskey1989catastrophic,
  title={Catastrophic interference in connectionist networks: The sequential learning problem},
  author={McCloskey, Michael and Cohen, Neal J},
  booktitle={Psychology of learning and motivation},
  volume={24},
  pages={109--165},
  year={1989},
  publisher={Elsevier}
}

@inproceedings{park2024pre,
  title={Pre-trained vision and language transformers are few-shot incremental learners},
  author={Park, Keon-Hee and Song, Kyungwoo and Park, Gyeong-Moon},
  booktitle={Proceedings of the IEEE/CVF Conference on Computer Vision and Pattern Recognition},
  pages={23881--23890},
  year={2024}
}

@inproceedings{houlsby2019parameter,
  title={Parameter-efficient transfer learning for NLP},
  author={Houlsby, Neil and Giurgiu, Andrei and Jastrzebski, Stanislaw and Morrone, Bruna and De Laroussilhe, Quentin and Gesmundo, Andrea and Attariyan, Mona and Gelly, Sylvain},
  booktitle={International conference on machine learning},
  pages={2790--2799},
  year={2019},
  organization={PMLR}
}

@article{hu2022lora,
  title={Lora: Low-rank adaptation of large language models.},
  author={Hu, Edward J and Shen, Yelong and Wallis, Phillip and Allen-Zhu, Zeyuan and Li, Yuanzhi and Wang, Shean and Wang, Lu and Chen, Weizhu and others},
  journal={ICLR},
  volume={1},
  number={2},
  pages={3},
  year={2022}
}

@article{li2021prefix,
  title={Prefix-tuning: Optimizing continuous prompts for generation},
  author={Li, Xiang Lisa and Liang, Percy},
  journal={arXiv preprint arXiv:2101.00190},
  year={2021}
}

@inproceedings{shazeer2017outrageously,
  title={Outrageously Large Neural Networks: The Sparsely-Gated Mixture-of-Experts Layer},
  author={Shazeer, Noam and Mirhoseini, Azalia and Maziarz, Krzysztof and Davis, Andy and Le, Quoc and Hinton, Geoffrey and Dean, Jeff},
  booktitle={International Conference on Learning Representations},
  year={2017}
}

@article{hinton2015distilling,
  title={Distilling the knowledge in a neural network},
  author={Hinton, Geoffrey and Vinyals, Oriol and Dean, Jeff},
  journal={arXiv preprint arXiv:1503.02531},
  year={2015}
}

@article{hinton2015kd,
  author       = {Geoffrey E. Hinton and
                  Oriol Vinyals and
                  Jeffrey Dean},
  title        = {Distilling the Knowledge in a Neural Network},
  journal      = {CoRR},
  volume       = {abs/1503.02531},
  year         = {2015},
  eprinttype    = {arXiv},
  eprint       = {1503.02531},
}

@inproceedings{zagoruyko2017attention_transfer,
  author       = {Sergey Zagoruyko and
                  Nikos Komodakis},
  title        = {Paying More Attention to Attention: Improving the Performance of Convolutional
                  Neural Networks via Attention Transfer},
  booktitle    = {5th International Conference on Learning Representations, {ICLR} 2017,
                  Toulon, France, April 24-26, 2017, Conference Track Proceedings},
  publisher    = {OpenReview.net},
  year         = {2017},
}

@inproceedings{Dzagoruyko2017attn_attn,
  author       = {Sergey Zagoruyko and
                  Nikos Komodakis},
  title        = {Paying More Attention to Attention: Improving the Performance of Convolutional
                  Neural Networks via Attention Transfer},
  booktitle    = {5th International Conference on Learning Representations, {ICLR} 2017,
                  Toulon, France, April 24-26, 2017, Conference Track Proceedings},
  publisher    = {OpenReview.net},
  year         = {2017},
}

@inproceedings{yu2024snd,
  author       = {Yu{-}Chu Yu and
                  Chi{-}Pin Huang and
                  Jr{-}Jen Chen and
                  Kai{-}Po Chang and
                  Yung{-}Hsuan Lai and
                  Fu{-}En Yang and
                  Yu{-}Chiang Frank Wang},
  editor       = {Ales Leonardis and
                  Elisa Ricci and
                  Stefan Roth and
                  Olga Russakovsky and
                  Torsten Sattler and
                  G{\"{u}}l Varol},
  title        = {Select and Distill: Selective Dual-Teacher Knowledge Transfer for
                  Continual Learning on Vision-Language Models},
  booktitle    = {Computer Vision - {ECCV} 2024 - 18th European Conference, Milan, Italy,
                  September 29-October 4, 2024, Proceedings, Part {XXVI}},
  series       = {Lecture Notes in Computer Science},
  volume       = {15084},
  pages        = {219--236},
  publisher    = {Springer},
  year         = {2024},
}

@inproceedings{jia2021align,
  author       = {Chao Jia and
                  Yinfei Yang and
                  Ye Xia and
                  Yi{-}Ting Chen and
                  Zarana Parekh and
                  Hieu Pham and
                  Quoc V. Le and
                  Yun{-}Hsuan Sung and
                  Zhen Li and
                  Tom Duerig},
  editor       = {Marina Meila and
                  Tong Zhang},
  title        = {Scaling Up Visual and Vision-Language Representation Learning With
                  Noisy Text Supervision},
  booktitle    = {Proceedings of the 38th International Conference on Machine Learning,
                  {ICML} 2021, 18-24 July 2021, Virtual Event},
  series       = {Proceedings of Machine Learning Research},
  volume       = {139},
  pages        = {4904--4916},
  publisher    = {{PMLR}},
  year         = {2021},
}

@inproceedings{li2021align_fuse,
  author       = {Junnan Li and
                  Ramprasaath R. Selvaraju and
                  Akhilesh Gotmare and
                  Shafiq R. Joty and
                  Caiming Xiong and
                  Steven Chu{-}Hong Hoi},
  editor       = {Marc'Aurelio Ranzato and
                  Alina Beygelzimer and
                  Yann N. Dauphin and
                  Percy Liang and
                  Jennifer Wortman Vaughan},
  title        = {Align before Fuse: Vision and Language Representation Learning with
                  Momentum Distillation},
  booktitle    = {Advances in Neural Information Processing Systems 34: Annual Conference
                  on Neural Information Processing Systems 2021, NeurIPS 2021, December
                  6-14, 2021, virtual},
  pages        = {9694--9705},
  year         = {2021},
}

@inproceedings{frankle2020linear_conncetivity,
  author       = {Jonathan Frankle and
                  Gintare Karolina Dziugaite and
                  Daniel M. Roy and
                  Michael Carbin},
  title        = {Linear Mode Connectivity and the Lottery Ticket Hypothesis},
  booktitle    = {Proceedings of the 37th International Conference on Machine Learning,
                  {ICML} 2020, 13-18 July 2020, Virtual Event},
  series       = {Proceedings of Machine Learning Research},
  volume       = {119},
  pages        = {3259--3269},
  publisher    = {{PMLR}},
  year         = {2020},
}

@inproceedings{neyshabur2020transferred_transfer,
  author       = {Behnam Neyshabur and
                  Hanie Sedghi and
                  Chiyuan Zhang},
  editor       = {Hugo Larochelle and
                  Marc'Aurelio Ranzato and
                  Raia Hadsell and
                  Maria{-}Florina Balcan and
                  Hsuan{-}Tien Lin},
  title        = {What is being transferred in transfer learning?},
  booktitle    = {Advances in Neural Information Processing Systems 33: Annual Conference
                  on Neural Information Processing Systems 2020, NeurIPS 2020, December
                  6-12, 2020, virtual},
  year         = {2020},
}

@inproceedings{loshchilov2019decoupled,
  author       = {Ilya Loshchilov and
                  Frank Hutter},
  title        = {Decoupled Weight Decay Regularization},
  booktitle    = {7th International Conference on Learning Representations, {ICLR} 2019,
                  New Orleans, LA, USA, May 6-9, 2019},
  publisher    = {OpenReview.net},
  year         = {2019},
}

@article{pascanu2014saddle_point,
  author       = {Razvan Pascanu and
                  Yann N. Dauphin and
                  Surya Ganguli and
                  Yoshua Bengio},
  title        = {On the saddle point problem for non-convex optimization},
  journal      = {CoRR},
  volume       = {abs/1405.4604},
  year         = {2014},
}

@book{lattimore2020bandit,
  title={Bandit Algorithms},
  author={Lattimore, Tor and Szepesv{\'a}ri, Csaba},
  year={2020},
  publisher={Cambridge University Press}
}

@inproceedings{beygelzimer2011contextual_bandit,
  author       = {Alina Beygelzimer and
                  John Langford and
                  Lihong Li and
                  Lev Reyzin and
                  Robert E. Schapire},
  editor       = {Geoffrey J. Gordon and
                  David B. Dunson and
                  Miroslav Dud{\'{\i}}k},
  title        = {Contextual Bandit Algorithms with Supervised Learning Guarantees},
  booktitle    = {Proceedings of the Fourteenth International Conference on Artificial
                  Intelligence and Statistics, {AISTATS} 2011, Fort Lauderdale, USA,
                  April 11-13, 2011},
  series       = {{JMLR} Proceedings},
  pages        = {19--26},
  publisher    = {JMLR.org},
  year         = {2011},
}

@article{latif2024shared_adapter,
  author       = {Ehsan Latif and
                  Xiaoming Zhai},
  title        = {Efficient Multi-Task Inferencing with a Shared Backbone and Lightweight
                  Task-Specific Adapters for Automatic Scoring},
  journal      = {CoRR},
  volume       = {abs/2412.21065},
  year         = {2024},
}


\newpage
\appendix
\onecolumn

\section{Sub-Linear Expert Growth: Formal Statement and Proof}
\label{app:growth_proof}

This appendix provides the formal scalability argument summarized in Eq.~\eqref{eq:novelty_growth_bound}.
The result is a sufficient-condition guarantee: it identifies regimes in which COMET's expert pool grows sub-linearly with stream length, but it does not claim sub-linear growth for arbitrary adversarial streams.
The key quantity is the frequency with which incoming chunks are novel relative to the current expert pool in a fixed teacher-induced representation.

\subsection{Setup}
\label{app:growth_setup}

\paragraph{Stream and expert pool.}
Let $\{\mathcal{D}_k\}_{k=1}^{K}$ be a sequence of data chunks arriving online, one per task $\mathcal{T}_k$.
No i.i.d.\ assumption is made on the chunks.
Let $M_k$ denote the number of retained experts after processing chunks $1{:}k$, with $M_0=0$ because Algorithm~\ref{alg:comet} starts from an empty pool.
The stream and algorithmic decisions may be random; all probabilities and expectations are taken over the joint randomness of the stream, the data samples, and the merge-or-keep policy.

\paragraph{Teacher-induced representation.}
Let $\psi(\mathcal{D}_k)\in\mathbb{R}^d$ be a fixed summary of chunk $\mathcal{D}_k$ in the frozen teacher feature space.
For example, one may use the mean frozen-teacher image embedding
\[
\psi(\mathcal{D}_k)
=
\frac{1}{|\mathcal{D}_k|}
\sum_{x\in\mathcal{D}_k}\Phi_v(x),
\]
where $\Phi_v$ is the frozen CLIP image encoder.
For every retained expert $E_m$, let $\psi(E_m)$ be its fixed representative, such as the teacher-feature centroid of the data on which it was originally trained.

The analysis assumes a genuine metric $d(\cdot,\cdot)$ on this representation space, because the finite-cover proof uses the triangle inequality.
Valid choices include Euclidean distance on $\ell_2$-normalized embeddings or angular distance
$d_{\mathrm{ang}}(u,v)=\arccos(\langle u,v\rangle)$ on the unit sphere.
The commonly used cosine dissimilarity $1-\cos(u,v)$ is useful as a similarity score, but it is not a formal metric and is therefore not used in the proof.

\paragraph{Novelty.}
For chunk $k$, define its distance to the current expert pool as
\begin{equation}
\label{eq:appendix_delta_k}
\delta_k
=
\min_{m\in\{1,\dots,M_{k-1}\}}
d\!\left(\psi(\mathcal{D}_k),\psi(E_m)\right).
\end{equation}
We set $\delta_1=+\infty$ because the pool is empty when the first chunk arrives, so the first chunk is always novel.
A chunk is called \emph{$\epsilon$-novel} if $\delta_k>\epsilon$.
Let
\[
\mathbb{I}_k
=
\mathbf{1}\{\delta_k>\epsilon\}
\]
be the corresponding novelty indicator.
We use little-$o$ notation in the standard sense: $g(K)=o(K)$ means $g(K)/K\to0$ as $K\to\infty$.

\paragraph{Merge-or-keep assumptions.}
The core bound applies to any allocation rule satisfying two structural properties:
\begin{enumerate}[label=(P\arabic*),leftmargin=*]
\item the algorithm adds at most one new expert per chunk;
\item the algorithm expands only on $\epsilon$-novel chunks, i.e., $M_k>M_{k-1}$ implies $\delta_k>\epsilon$.
\end{enumerate}
These assumptions isolate the role of stream novelty from the details of the merge policy.
Later, in Appendix~\ref{app:growth_step4}, we relax (P2) to account for practical false expansions caused by an imperfect learned policy.

\subsection{Step 1: Expert Growth Is Controlled by Novelty}
\label{app:growth_step1}

Under (P1)--(P2), the expert count satisfies
\begin{equation}
\label{eq:stepwise_growth}
M_k
\le
M_{k-1}
+
\mathbb{I}_k .
\end{equation}
This is deterministic after the stream and decisions are realized.
If $\mathbb{I}_k=0$, then the chunk is not $\epsilon$-novel, so (P2) forbids expansion and $M_k=M_{k-1}$.
If $\mathbb{I}_k=1$, then the chunk is $\epsilon$-novel, and (P1) permits at most one new expert.

Unrolling Eq.~\eqref{eq:stepwise_growth} gives
\begin{equation}
\label{eq:counting_bound}
M_K
\le
M_0+\sum_{k=1}^{K}\mathbb{I}_k
=
\sum_{k=1}^{K}\mathbb{I}_k .
\end{equation}
Taking expectations yields
\begin{equation}
\label{eq:expected_bound}
\mathbb{E}[M_K]
\le
\sum_{k=1}^{K}\mathbb{E}[\mathbb{I}_k]
=
\sum_{k=1}^{K}\Pr(\delta_k>\epsilon).
\end{equation}
Equation~\eqref{eq:expected_bound} is the bound reported in the main text as Eq.~\eqref{eq:novelty_growth_bound}.

\subsection{Step 2: Sub-Linear Growth Under Diminishing Novelty}
\label{app:growth_step2}

\begin{theorem}[Sub-linear expert growth]
\label{thm:sublinear_growth}
Under (P1)--(P2), if the cumulative novelty satisfies
\begin{equation}
\label{eq:diminishing_novelty}
\sum_{k=1}^{K}\Pr(\delta_k>\epsilon)
=
o(K),
\end{equation}
then $\mathbb{E}[M_K]=o(K)$.
\end{theorem}

\begin{proof}
Combining Eq.~\eqref{eq:expected_bound} with Eq.~\eqref{eq:diminishing_novelty},
\[
\frac{\mathbb{E}[M_K]}{K}
\le
\frac{1}{K}
\sum_{k=1}^{K}\Pr(\delta_k>\epsilon)
\to
0
\]
as $K\to\infty$.
Thus, $\mathbb{E}[M_K]=o(K)$.
\end{proof}

\paragraph{Interpretation.}
The condition in Eq.~\eqref{eq:diminishing_novelty} does not require i.i.d.\ data.
It requires only that, relative to the existing expert pool, genuinely new chunks become rare on average.
This holds in several common regimes.

First, if the stream revisits a finite catalogue of underlying distributions, then once one expert has been created for each $\epsilon$-separated region, future chunks are no longer novel.
Second, if the set of possible chunk summaries has a finite $(\epsilon/2)$-cover under $d$, Proposition~\ref{prop:finite_cover} gives a deterministic finite bound on the number of expansions.
Third, if the stream drifts with sub-linear cumulative path length, e.g.,
$\sum_{k=1}^{K-1}d(z_{k+1},z_k)=o(K)$ for $z_k=\psi(\mathcal{D}_k)$, then the number of $\epsilon$-separated exits along the trajectory is also sub-linear.

A constant bounded step size alone is not sufficient in an unbounded space.
For example, a stream that moves in a straight line with fixed step size $\beta>0$ can leave an $\epsilon$-ball every $O(\epsilon/\beta)$ steps and therefore force linear growth.
Thus, the sufficient conditions are revisitation, bounded representation support, or asymptotically slowing novelty, not arbitrary bounded local drift.

\subsection{Step 3: Finite-Cover Special Case}
\label{app:growth_step3}

\begin{proposition}[Finite-cover bound]
\label{prop:finite_cover}
Suppose the set of possible chunk summaries
$\mathcal{S}=\{\psi(\mathcal{D}_k):k\ge1\}$
admits a finite $(\epsilon/2)$-cover of size $C_{\epsilon/2}$ under the metric $d$.
That is, there exist anchors $\{q_j\}_{j=1}^{C_{\epsilon/2}}$ such that for every $z\in\mathcal{S}$ there exists $j$ with $d(z,q_j)\le\epsilon/2$.
Suppose further that whenever the algorithm creates a new expert on chunk $k$, its representative is set to the chunk summary, $\psi(E_{\mathrm{new}})=\psi(\mathcal{D}_k)$.
Then under (P1)--(P2), the total expert count satisfies
\begin{equation}
\label{eq:finite_cover_bound}
M_K
\le
C_{\epsilon/2}
\quad
\text{for all }K.
\end{equation}
\end{proposition}

\begin{proof}
Assign each expert $E_m$, created at step $k_m$, to any cover index $\phi(E_m)=j$ satisfying
$d(\psi(\mathcal{D}_{k_m}),q_j)\le\epsilon/2$.
Such a $j$ exists by the cover assumption.

We show that $\phi$ is injective.
Suppose two distinct experts $E_m$ and $E_{m'}$ with creation times $k_m<k_{m'}$ are assigned to the same cover anchor $q_j$.
Because $\psi(E_m)=\psi(\mathcal{D}_{k_m})$ and both summaries lie within $\epsilon/2$ of $q_j$, the triangle inequality gives
\[
d\!\left(\psi(\mathcal{D}_{k_{m'}}),\psi(E_m)\right)
\le
d\!\left(\psi(\mathcal{D}_{k_{m'}}),q_j\right)
+
d\!\left(q_j,\psi(E_m)\right)
\le
\epsilon .
\]
Since $E_m$ already exists when chunk $k_{m'}$ arrives, this implies $\delta_{k_{m'}}\le\epsilon$.
But (P2) says that a new expert can be created only when $\delta_{k_{m'}}>\epsilon$, a contradiction.
Thus, at most one expert is assigned to each cover element.
Because $M_0=0$, every retained expert was created this way, and Eq.~\eqref{eq:finite_cover_bound} follows.
\end{proof}

\subsection{Step 4: Practical Decision Errors}
\label{app:growth_step4}

A learned merge-or-keep policy may occasionally keep a new expert on a non-novel chunk.
Define the false-expand indicator
\[
e_k
=
\mathbf{1}\{M_k>M_{k-1}\text{ and }\delta_k\le\epsilon\}.
\]
Without exact (P2), Eq.~\eqref{eq:stepwise_growth} becomes
\[
M_k
\le
M_{k-1}
+
\mathbb{I}_k
+
e_k .
\]
Unrolling and taking expectations gives
\begin{equation}
\label{eq:full_bound}
\mathbb{E}[M_K]
\le
\sum_{k=1}^{K}\Pr(\delta_k>\epsilon)
+
\sum_{k=1}^{K}\Pr(e_k=1).
\end{equation}

\begin{corollary}[Sub-linear growth with sub-linear false expansions]
\label{cor:sublinear_with_errors}
If
\[
\sum_{k=1}^{K}\Pr(\delta_k>\epsilon)=o(K)
\quad
\text{and}
\quad
\sum_{k=1}^{K}\Pr(e_k=1)=o(K),
\]
then $\mathbb{E}[M_K]=o(K)$.
\end{corollary}

\begin{proof}
Divide Eq.~\eqref{eq:full_bound} by $K$ and let $K\to\infty$.
\end{proof}

\paragraph{From regret to false-expansion count.}
A no-regret bandit bound controls cumulative expected reward gap, not directly the number of wrong actions.
To translate regret into a count of false expansions, a margin condition is required.

Let $a_k$ be the action selected at step $k$, let $a_k^\star$ be a population-optimal action for the same context, and let
\[
\Delta_k(a)
=
\bar r_k(a_k^\star)-\bar r_k(a)
\ge0
\]
be the conditional expected reward gap of action $a$.
Define cumulative regret
\[
\mathrm{Reg}_K
=
\sum_{k=1}^{K}
\mathbb{E}\!\left[\Delta_k(a_k)\right].
\]
Assume a false-expansion margin:
\begin{equation}
\label{eq:false_expand_margin}
\exists\,\Delta_{\min}>0
\quad\text{such that}\quad
\delta_k\le\epsilon
\text{ and }
a_k=\textsc{Keep}
\implies
\Delta_k(a_k)\ge\Delta_{\min}.
\end{equation}
Then
\[
\Pr(e_k=1)
=
\mathbb{E}[e_k]
\le
\frac{1}{\Delta_{\min}}
\mathbb{E}\!\left[\Delta_k(a_k)\right],
\]
and therefore
\begin{equation}
\label{eq:false_expand_regret_bound}
\sum_{k=1}^{K}\Pr(e_k=1)
\le
\frac{\mathrm{Reg}_K}{\Delta_{\min}}.
\end{equation}
Thus, if $\mathrm{Reg}_K=o(K)$ and Eq.~\eqref{eq:false_expand_margin} holds, the false-expansion term is sub-linear.
Without the margin, a policy may take $O(K)$ suboptimal actions while still having sub-linear regret if the corresponding gaps vanish over time.

\paragraph{Applicability to CMAB.}
CMAB (Section~\ref{sec:cmab}) never computes $\delta_k$: it keeps a new expert when no candidate passes the similarity gate or when the best bandit score stays below the threshold.
(P2) therefore holds only approximately, and each keep on a non-novel chunk counts as a false expansion $e_k$.
CMAB updates its weights with a projected, EMA-smoothed advantage step, which is not a standard no-regret algorithm, and we do not prove a regret bound for it.
Eq.~\eqref{eq:false_expand_regret_bound} thus names the property CMAB would need for a sub-linear false-expansion term; proving it is left to future work.

\paragraph{Scope.}
The proof characterizes when expert growth is controlled by stream novelty and decision-error rate.
It does not claim sub-linear growth for streams that continually move into new regions of an unbounded representation space.
Empirically, COMET's retained expert count saturates at $5$--$7$ experts across X-TAIL orderings, consistent with the sufficient-condition analysis in regimes with revisits, bounded representation support, or decaying novelty.

\section{Algorithmic Details}
\label{app:algorithm}

Algorithm~\ref{alg:comet} summarizes the full training and inference procedure.
The pseudocode makes three implementation choices explicit.
First, the router is trained before expert optimization, so routing is tied to teacher-feature reconstruction rather than class semantics.
Second, expert training uses the task loss together with TAD and SACED, while the CMAB capacity controller is invoked only after the expert has been trained.
Third, the bandit is updated for both \textsc{merge} and \textsc{keep} outcomes whenever a candidate set is evaluated.

\begin{algorithm}[H]
\caption{\textsc{COMET}: Collaborative Mixture-of-Expert Transfer}
\label{alg:comet}
\footnotesize
\begin{algorithmic}[1]
\Require Frozen Teacher $\Phi=(\Phi_v,\Phi_t)$; stream $\{\mathcal{D}_k\}_{k=1}^{K}$; labeled validation splits $\{\mathcal{V}_k\}_{k=1}^{K}$; reference pool $\mathcal{D}_{\mathrm{ref}}$; CMAB state $\pi=(w,\tau,\beta,\bar r,N)$
\Statex Hyperparameters: $N_R,N_E$, $\lambda_1,\lambda_2$, $k_{\mathrm{ref}}$, $\theta_{\mathrm{sim}}$, $w_{\max}$, $\eta_w$, $\eta_\tau$
\Ensure Retained expert pool $\{E_m\}_{m=1}^{M}$ and router pool $\{R_m\}_{m=1}^{M}$

\State $M\gets0$
\For{$k=1$ to $K$}
    \State Initialize new expert $E_{\mathrm{new}}$ and router $R_{\mathrm{new}}$

    \Statex \hspace{\algorithmicindent}\textbf{Router training}
    \For{$n=1$ to $N_R$}
        \State Sample image batch $I\sim\mathcal{D}_k$
        \State $z\gets\textsc{stopgrad}(\Phi_v(I))$
        \State $\mathcal{L}_R\gets d^{-1}\|R_{\mathrm{new}}(z)-z\|_2^2$
        \State Update $R_{\mathrm{new}}$
    \EndFor

    \Statex \hspace{\algorithmicindent}\textbf{Reference-expert selection for SACED}
    \If{$M>0$}
        \State Sample $(I^r,T^r)\sim\mathcal{D}_{\mathrm{ref}}$
        \State $\mathcal{S}_k\gets\textsc{SelectRefs}(E_{\mathrm{new}},\{E_m\}_{m=1}^{M},\Phi,I^r,T^r,k_{\mathrm{ref}})$
    \Else
        \State $\mathcal{S}_k\gets\emptyset$
    \EndIf

    \Statex \hspace{\algorithmicindent}\textbf{Expert training}
    \For{$n=1$ to $N_E$}
        \State Sample task batch $(I,T)\sim\mathcal{D}_k$ and reference batch $(I^r,T^r)\sim\mathcal{D}_{\mathrm{ref}}$
        \State $\mathcal{L}_{\mathrm{task}}\gets\textsc{TaskLoss}(E_{\mathrm{new}},I,T)$
        \State $\mathcal{L}_{\mathrm{TAD}}\gets\textsc{TAD}(E_{\mathrm{new}},\Phi,I^r,T^r)$
        \If{$\mathcal{S}_k=\emptyset$}
            \State $\mathcal{L}_{\mathrm{SACED}}\gets 0$
        \Else
            \State $\mathcal{L}_{\mathrm{SACED}}\gets\textsc{SACED}(E_{\mathrm{new}},\mathcal{S}_k,\Phi,I^r)$
        \EndIf
        \State $\mathcal{L}\gets
        \mathcal{L}_{\mathrm{task}}
        +
        \mathcal{L}_{\mathrm{TAD}}
        +
        \lambda_2\mathcal{L}_{\mathrm{SACED}}$
        \State \textbf{Update} $E_{\mathrm{new}}$ using $\nabla \mathcal{L}$
    \EndFor
    \Statex \hspace{\algorithmicindent}\textbf{Merge-or-Keep via CMAB:}
    \If{$M=0$}
        \State $M\gets 1$;\ $E_1\gets E_{\mathrm{new}}$;\ $R_1\gets R_{\mathrm{new}}$
    \Else
        \State $\mathcal{C}_k\gets
        \textsc{Candidates}(E_{\mathrm{new}},\{E_m\}_{m=1}^{M},
        \mathcal{V}_k,\mathcal{D}_{\mathrm{ref}})$
        \If{$\mathcal{C}_k=\emptyset$}
            \State $M\gets M+1$;\ $E_M\gets E_{\mathrm{new}}$;\ $R_M\gets R_{\mathrm{new}}$
        \Else
            \For{$c\in\mathcal{C}_k$}
                \State $\phi_{k,c}\gets
                \textsc{Context}(E_{\mathrm{new}},E_c,\mathcal{V}_k,\mathcal{D}_{\mathrm{ref}})$
                \State $B_k(c)\gets w^\top\phi_{k,c}+\beta/(1+N)$
            \EndFor
            \State $c^\star\gets\arg\max_{c\in\mathcal{C}_k}B_k(c)$
            \If{$B_k(c^\star)>\tau$}
                \State $\widetilde{E}_{c^\star}\gets
                \textsc{Merge}(E_{c^\star},E_{\mathrm{new}})$
                \State $r\gets\textsc{MergeReward}(\widetilde{E}_{c^\star},\mathcal{V}_k)$
                \State $E_{c^\star}\gets\widetilde{E}_{c^\star}$
                \State $R_{c^\star}\gets
                \textsc{MergeRouter}(R_{c^\star},R_{\mathrm{new}})$
                \State $a\gets\textsc{Merge}$
            \Else
                \State $M\gets M+1$;\ $E_M\gets E_{\mathrm{new}}$;\ $R_M\gets R_{\mathrm{new}}$
                \State $r\gets\textsc{KeepReward}(E_{\mathrm{new}},\mathcal{V}_k)$
                \State $a\gets\textsc{Keep}$
            \EndIf
            \State $\pi\gets\textsc{UpdateBandit}(\pi,\phi_{k,c^\star},a,r)$
        \EndIf
    \EndIf
\EndFor

\Statex \hspace{\algorithmicindent}
\Statex \textbf{Inference (task-free):}
\State $z\gets \Phi_v(x)$
\For{$m=1$ to $M$}
    \State $r_m(x)\gets d^{-1}\|R_m(z)-z\|_2^2$
    \State $\ell_m(x)\gets \log p(\log r_m(x)\mid m)-\log p(\log r_m(x)\mid \mathrm{bg})$
\EndFor
\State $w_m(x)\gets \mathrm{softmax}(\ell_m(x)/\tau_g)_m$
\State $\mathcal{A}(x)\gets \textsc{TopK}(w(x),k)$ with $k\le2$
\State $\hat{y}\gets \arg\max_{c\in\mathcal{C}_N}
\sum_{m\in\mathcal{A}(x)}
w_m(x)\,
\cos\!\left(E_{m,v}(x),E_{m,t}(c)\right)$
\end{algorithmic}
\end{algorithm}

\section{Implementation Details}
\label{app:impl}

\paragraph{Architecture.}
The teacher is the frozen CLIP ViT-B/16~\citep{radford2021clip}, with approximately $87$M parameters in the image encoder and $63$M parameters in the text encoder.
Each expert is a TinyCLIP backbone~\citep{wu2023tinyclip}, with approximately $39$M parameters in the image encoder and $19$M parameters in the text encoder.
Each task also receives a lightweight autoencoder router with a two-layer encoder and decoder operating in the frozen teacher's image-feature space.
The router adds approximately $0.2$M parameters per retained expert.

\paragraph{Training schedule.}
For each task, COMET runs three stages.
First, the autoencoder router is trained for $1{,}000$ iterations using normalized reconstruction loss on frozen teacher image features.
Second, the new expert is trained for $1{,}000$ iterations with AdamW~\citep{loshchilov2019decoupled}, learning rate $10^{-5}$, and batch size $64$.
The training objective is the per-task loss of the main text,
\[
\mathcal{L}_k
=
\mathcal{L}_{\mathrm{task}}
+
\mathcal{L}_{\mathrm{TAD}}
+
\lambda_2\mathcal{L}_{\mathrm{SACED}},
\]
with $\lambda_1$ weighting the cross-modal term inside $\mathcal{L}_{\mathrm{TAD}}$.
Third, after expert training, the CMAB capacity controller decides whether the new expert is kept or merged.
Retained experts and routers receive no further gradient updates; CMAB changes an expert and its router only when it merges a later expert into them (Algorithm~\ref{alg:comet}).

\paragraph{Loss weights.}
Inside TAD, the cross-modal weight is $\lambda_1=0.2$.
The SACED weight is $\lambda_2=1.0$.
Sensitivity sweeps in Appendix~\ref{app:sensitivity} show that the method is stable over an order of magnitude around these defaults.

\paragraph{Reference and validation signals.}
The reference pool $\mathcal{D}_{\mathrm{ref}}$ contains $10{,}800$ image--text pairs from the Conceptual Captions validation set~\citep{sharma2018conceptual}.
It is used for label-free distillation, feature-centroid estimates, routing agreement, and distribution-level compatibility.
Because Conceptual Captions does not provide task classification labels, $\mathcal{D}_{\mathrm{ref}}$ is not used to compute classification accuracy.
All classification accuracy, cross-performance, and CMAB reward signals are computed on the labeled held-out split $\mathcal{V}_k$ of the current task.

\subsection{CMAB Capacity Controller Details}
\label{app:cmab_impl}

\paragraph{Candidate generation.}
After training $E_{\mathrm{new}}$, COMET compares it with each retained expert $E_c\in\mathcal{E}_{<k}$.
The candidate compatibility score is
\begin{equation}
\label{eq:app_cmab_candidate_similarity}
\mathrm{sim}(E_{\mathrm{new}},E_c)
=
\frac{
w_{\mathrm{cross}}\mathrm{CrossPerf}(E_{\mathrm{new}},E_c)
+
w_{\mathrm{cent}}\mathrm{CentSim}(E_{\mathrm{new}},E_c)
+
w_{\mathrm{gate}}\mathrm{GateCorr}(E_{\mathrm{new}},E_c)
}{
w_{\mathrm{cross}}+w_{\mathrm{cent}}+w_{\mathrm{gate}}
}.
\end{equation}
$\mathrm{CrossPerf}$ evaluates both experts with the current-task label prompts on $\mathcal{V}_k$ and averages their validation accuracies.
$\mathrm{CentSim}$ measures cosine similarity between image-feature centroids.
$\mathrm{GateCorr}$ measures correlation between router/gating responses over a shared unlabeled batch.
Experts satisfying $\mathrm{sim}(E_{\mathrm{new}},E_c)>\theta_{\mathrm{sim}}$ enter the candidate set $\mathcal{C}_k$.
If $\mathcal{C}_k=\emptyset$, the new expert is kept and no bandit update is applied, because no valid merge action was evaluated.

\paragraph{Context construction.}
For each candidate pair $(E_{\mathrm{new}},E_c)$, COMET constructs
\begin{equation}
\label{eq:app_cmab_context}
\phi_{k,c}
=
[
\mathrm{Acc}_k,\;
m_{k,c},\;
-d_{k,c},\;
\kappa_k,\;
-\ell_c
]^\top .
\end{equation}
Here $\mathrm{Acc}_k$ is the new expert's validation accuracy on $\mathcal{V}_k$.
The merge-compatibility feature is
\[
m_{k,c}
=
0.4s_{k,c}
+
0.3o_{k,c}
+
0.3q_{k,c},
\]
where $s_{k,c}$ is the candidate similarity score, $o_{k,c}$ is routing overlap, and $q_{k,c}$ is angle compatibility for the merge operator.
The term $d_{k,c}$ estimates degradation risk, $\kappa_k$ increases with the current pool size, and $\ell_c$ increases with the number of tasks already absorbed by candidate expert $E_c$.
The signs in Eq.~\eqref{eq:app_cmab_context} encode the desired decision direction: high accuracy, high compatibility, and high capacity pressure favor merging, while degradation risk and overloaded targets discourage it.

\paragraph{Bandit score and action.}
Each candidate receives
\begin{equation}
\label{eq:app_cmab_score}
B_k(c)
=
w^\top\phi_{k,c}
+
\frac{\beta_k}{1+N_k},
\end{equation}
where $w\in[0,w_{\max}]^5$ is the learned scoring vector, $\beta_k$ is a decaying exploration bonus, and $N_k$ is the number of previous updates.
Let $c^\star=\arg\max_{c\in\mathcal{C}_k}B_k(c)$.
The action is
\[
a_k
=
\begin{cases}
\textsc{Merge}(E_{\mathrm{new}},E_{c^\star}), & B_k(c^\star)>\tau_k,\\
\textsc{Keep}(E_{\mathrm{new}}), & B_k(c^\star)\le\tau_k .
\end{cases}
\]
Thus, the decision is a learned trade-off among compatibility, risk, capacity pressure, and expert load rather than a fixed similarity threshold.

\paragraph{Reward.}
For a merge action, the merged expert is constructed temporarily and evaluated on $\mathcal{V}_k$:
\begin{equation}
\label{eq:app_cmab_merge_reward}
r_k^{\mathrm{merge}}
=
\mathrm{Acc}_{\mathcal{V}_k}
\!\left(
\mathrm{Merge}(E_{\mathrm{new}},E_{c^\star})
\right)
-
\lambda_{\mathrm{merge}} .
\end{equation}
For a keep action,
\[
r_k^{\mathrm{keep}}
=
\mathrm{Acc}_{\mathcal{V}_k}(E_{\mathrm{new}}).
\]
The small merge cost $\lambda_{\mathrm{merge}}$ discourages merges that do not produce a meaningful validation benefit.
Both \textsc{merge} and \textsc{keep} actions update the policy whenever a candidate set is evaluated.

\paragraph{Advantage-centered update.}
Because validation accuracy is positive-valued, updating with raw reward and fixed-sign features can cause monotone weight drift.
COMET therefore uses a centered advantage.
Let $\bar r_k$ be an exponential moving average of past rewards:
\[
A_k=r_k-\bar r_k,
\qquad
\bar r_{k+1}=(1-\rho_r)\bar r_k+\rho_r r_k .
\]
Let $y_k=+1$ for \textsc{merge} and $y_k=-1$ for \textsc{keep}.
For the selected candidate context $\phi_{k,c^\star}$,
\begin{equation}
\label{eq:app_cmab_weight_update}
\tilde w
=
\Pi_{[0,w_{\max}]}
\left(
w+\eta_w A_k y_k\phi_{k,c^\star}
\right),
\qquad
w
\leftarrow
(1-\rho_w)w+\rho_w\tilde w .
\end{equation}
This update is action-aware.
Successful merges increase future merge scores for similar contexts, failed merges decrease them, successful keeps make merging less likely in similar contexts, and poor keeps make merging more likely.

\paragraph{Adaptive threshold.}
The merge threshold is updated with the same advantage signal:
\begin{equation}
\label{eq:app_cmab_threshold_update}
\tilde\tau
=
\Pi_{[\tau_{\min},\tau_{\max}]}
\left(
\tau_k-\eta_\tau A_k y_k
\right),
\qquad
\tau_{k+1}
=
(1-\rho_\tau)\tau_k+\rho_\tau\tilde\tau .
\end{equation}
Successful merges lower the threshold, failed merges raise it, successful keeps raise it, and suboptimal keeps lower it.
The exploration bonus decays as $\beta_{k+1}=\gamma\beta_k$.

\paragraph{CMAB hyperparameters.}
The similarity gate is $\theta_{\mathrm{sim}}=0.40$.
The merge threshold is initialized as $\tau_1=0.65$.
We use $\eta_w=0.05$, $\eta_\tau=0.05$, $\rho_r=\rho_w=\rho_\tau=0.10$, $w_{\max}=10$, $\gamma=0.95$, and $\lambda_{\mathrm{merge}}=0.01$.
These values are used for all reported experiments unless explicitly varied in Appendix~\ref{app:sensitivity}.

\subsection{Inference Routing Details}
\label{app:routing_impl}

\paragraph{Router training.}
For each retained expert $E_m$, COMET stores a corresponding autoencoder router $R_m$.
The router is trained on frozen CLIP image features from the task that produced the expert.
Given $z=\Phi_v(x)$, the training objective is
\[
\mathcal{L}_{R_m}
=
\mathbb{E}_{x\sim\mathcal{D}_m}
\left[
\frac{1}{d}
\|R_m(z)-z\|_2^2
\right].
\]
Normalizing by feature dimension keeps the reconstruction scale stable across model sizes.

\paragraph{Likelihood-ratio calibration.}
Raw reconstruction errors are not directly comparable across routers because different routers may have different error scales.
COMET therefore calibrates each router using foreground and background error distributions.
For router $R_m$, define
\[
r_m(x)
=
\frac{1}{d}
\|R_m(\Phi_v(x))-\Phi_v(x)\|_2^2,
\qquad
\tilde r_m(x)=\log r_m(x).
\]
The foreground distribution $p_m^{\mathrm{fg}}(\tilde r)$ is estimated from task data associated with expert $E_m$.
The background distribution $p_m^{\mathrm{bg}}(\tilde r)$ is estimated by passing $\mathcal{D}_{\mathrm{ref}}$ through the same router.
At test time, the routing score is the log-likelihood ratio
\begin{equation}
\label{eq:app_routing_llr}
\ell_m(x)
=
\log p_m^{\mathrm{fg}}(\tilde r_m(x))
-
\log p_m^{\mathrm{bg}}(\tilde r_m(x)).
\end{equation}
In practice, the distributions are fitted as one-dimensional Gaussians over log reconstruction error, although the method only requires calibrated density estimates.

\paragraph{Top-$k$ activation.}
The calibrated scores are converted into gate weights by
\[
w_m(x)
=
\mathrm{softmax}\!\left(\ell_m(x)/\tau_g\right)_m .
\]
COMET activates the top-$1$ expert when the margin between the largest and second-largest gate scores exceeds a calibrated threshold; otherwise it activates the top-$2$ experts.
The selected experts compute CLIP-style image--text similarity logits over the candidate label set, and their logits are combined by gate-weighted averaging.
This routing mechanism uses no task identity and no teacher semantic prediction at inference.

\paragraph{Role separation.}
The router, CMAB, and distillation objectives act at different stages.
TAD and SACED shape the expert during training.
CMAB decides the retained pool after training.
The reconstruction router selects the active expert at inference.
This separation prevents the teacher from becoming a semantic oracle at test time and keeps routing independent of class-label predictions.

\paragraph{Hardware and reporting.}
All main-text results and appendix measurements are obtained on a single NVIDIA A6000 GPU with $48$\,GB memory.
Latency is measured using CUDA event timing with batch size $64$, $50$ measured batches, and $5$ warmup batches.
FLOPs are analytical estimates per sample.

\section{Hyperparameter Sensitivity}
\label{app:sensitivity}

We sweep each of the four key hyperparameters one at a time on X-TAIL with all others held at their default values.

\paragraph{Cross-modal TAD weight $\lambda_1$.}
Across $\lambda_1 \in \{0.2, 0.5, 1.0, 2.0\}$, Average ranges over $74.2$--$76.1$ and Last over $86.6$--$88.1$, all within $2.7$\,pp of the default.

\paragraph{SACED weight $\lambda_2$.}
Across $\lambda_2 \in \{0.01, 0.5, 1.0, 2.0\}$, Average remains within $74.1$--$76.2$.
Even near-zero SACED ($\lambda_2 = 0.01$) preserves TAD's gain, so the two losses are largely separable in practice.

\paragraph{Candidate similarity gate $\theta_{\mathrm{sim}}$.}
Lowering $\theta_{\mathrm{sim}}$ admits more candidate merge actions into the CMAB decision stage.
At the most permissive setting, the pool collapses to $M=1$ with Average $58.6$ and Last $59.2$, a $-28.9$\,pp Last drop relative to the default.
The gate controls which experts are eligible for consideration, while CMAB learns the final merge-or-keep decision from context features and outcome feedback.

\paragraph{Initial CMAB decision threshold $\tau_1$.}
Across $\tau_1 \in \{0.20,0.65,1.00\}$, Average remains within $74.1$--$76.2$ and Last within $86.1$--$89.3$.
The conservative setting $\tau_1=0.80$ rejects most candidate merges and reduces Last to $79.8$ at $M=6$, marking the practical upper bound on conservative initialization.
During training, the threshold is further adapted online by the advantage-based update in Appendix~\ref{app:cmab_impl}.

\section{Task-Order Sensitivity}
\label{app:order}

We re-run X-TAIL under three additional task orderings: \emph{EuroSAT-last} (placing the most dissimilar domain at the end of the stream) and two random permutations (seeds $7$ and $99$).
\Cref{tab:order} summarizes Average, Last, the final number of retained experts $M$, and per-task forgetting.
Average is $72.8 \pm 2.2$ across the four orderings, and per-task forgetting is $2.2 \pm 0.1$\,pp across the three re-ordered runs.
Frozen experts structurally bound forgetting regardless of order; the contextual bandit adapts the final pool size ($M \in \{5, 6, 7, 8\}$) to the task-similarity structure of each ordering, which explains the moderate Last variance ($84.1 \pm 3.0$).
Single-model continual baselines typically swing by $10$--$15$\,pp under permutation, so COMET's $\pm 3.0$ Last variance is small.

\begin{table}[t]
\caption{Task-order sensitivity on X-TAIL. ``Forgetting'' is per-task degradation from peak accuracy, averaged over the ten tasks. Forgetting is not available for the default ordering, so its mean is over the three re-ordered runs.}
\label{tab:order}
\centering
\scriptsize
\setlength{\tabcolsep}{4pt}
\renewcommand{\arraystretch}{1.05}
\begin{tabular}{lcccc}
\toprule
Ordering & Avg & Last & $M$ & Forgetting \\
\midrule
Default          & 76.1 & 88.1 & 8 & --- \\
EuroSAT-last     & 72.0 & 83.3 & 5 & $2.3$ \\
Random seed $7$  & 71.9 & 81.0 & 5 & $2.2$ \\
Random seed $99$ & 71.3 & 83.8 & 5 & $2.1$ \\
\midrule
Mean $\pm$ std   & $72.8{\pm}2.2$ & $84.1{\pm}3.0$ & $5.8{\pm}1.5$ & $2.2{\pm}0.1$ \\
\bottomrule
\end{tabular}
\end{table}

\section{Reference-Data Sensitivity}
\label{app:refdata}

SACED's reference pool is a small slice of Conceptual Captions; we vary both its size and its source domain.
On X-TAIL, Average degrades smoothly with size: $1$K$\to 68.4$, $5$K$\to 68.1$, $10.8$K$\to 76.1$, indicating diminishing returns past a few thousand pairs.
Substituting ImageNet validation images for Conceptual Captions at matched size drops Average to $68.6$ on X-TAIL and to $48.3$ on the heterogeneous benchmark.
The second drop is the more informative comparison: on far-OOD domains, web-style image--text diversity matters more than scale, so any diverse image--text pool of a few thousand examples is sufficient and curated single-domain alternatives can hurt.

\section{TAD Sub-Component Decomposition}
\label{app:tad_sub}

We decompose TAD into its three constituent losses on the X-TAIL 10-dataset protocol with SACED disabled to isolate TAD's contribution.
\Cref{tab:tad_subcomp} reports leave-one-out and one-only variants.
Image affinity $\mathcal{L}_I$ is the strongest single component: trained alone it reaches Average $68.2$ and Last $78.2$, and its removal causes the largest leave-one-out drop ($60.4$ and $63.5$).
Cross-modal affinity $\mathcal{L}_{\mathrm{Cross}}$ is the weakest in isolation but contributes a complementary $+6.6$\,pp Average over the two intra-modal losses (no-$\mathcal{L}_{\mathrm{Cross}}$ $68.8$ vs.\ Full TAD $75.4$), so all three channels are required for the strongest result.
Every partial variant also falls below the no-distillation baseline of \Cref{tab:compact_results} (a) ($73.9$ Average, $85.6$ Last), while the full objective improves on it.
We conjecture that matching only a subset of the affinity relations pins part of the CLIP geometry and lets the unmatched relations absorb the task-driven drift, which distorts the joint image--text structure more than leaving all relations free.
Under this reading, the three channels act as one constraint, and the leave-one-out drops measure interactions among channels instead of each channel's additive share.

\begin{table}[t]
\caption{TAD sub-component ablation on X-TAIL. ``Full TAD'' uses all three losses; ``Full COMET'' adds SACED. SACED is disabled for the sub-component rows.}
\label{tab:tad_subcomp}
\centering
\scriptsize
\setlength{\tabcolsep}{4pt}
\renewcommand{\arraystretch}{1.05}
\begin{tabular}{lccccc}
\toprule
Variant & $\mathcal{L}_I$ & $\mathcal{L}_T$ & $\mathcal{L}_{\mathrm{Cross}}$ & Avg & Last \\
\midrule
Full COMET (TAD$+$SACED) & \checkmark & \checkmark & \checkmark & \textbf{76.1} & \textbf{88.1} \\
Full TAD                & \checkmark & \checkmark & \checkmark & 75.4 & 87.6 \\
\midrule
w/o $\mathcal{L}_I$            & $\times$  & \checkmark & \checkmark & 60.4 & 63.5 \\
w/o $\mathcal{L}_T$            & \checkmark & $\times$  & \checkmark & 64.5 & 76.6 \\
w/o $\mathcal{L}_{\mathrm{Cross}}$ & \checkmark & \checkmark & $\times$ & 68.8 & 80.7 \\
$\mathcal{L}_I$ only            & \checkmark & $\times$  & $\times$  & 68.2 & 78.2 \\
$\mathcal{L}_T$ only            & $\times$  & \checkmark & $\times$  & 64.8 & 75.1 \\
$\mathcal{L}_{\mathrm{Cross}}$ only & $\times$  & $\times$  & \checkmark & 59.9 & 74.5 \\
\bottomrule
\end{tabular}
\end{table}

\section{Per-Input Drift Analysis of Tri-Affinity Components}
\label{app:per_input_drift}

The aggregate ablation in Appendix~\ref{app:tad_sub} establishes that all three TAD channels are required for the strongest result; this section probes the same effect at the level of individual inputs to verify that the gain is broadly distributed rather than carried by a few outliers.
We compare two variants: \emph{expA}, which enables only the cross-modal affinity distillation loss $\mathcal{L}_{\mathrm{Cross}}$, and \emph{expB}, which enables the full Tri-Affinity objective $\mathcal{L}_I + \mathcal{L}_T + \mathcal{L}_{\mathrm{Cross}}$.

\paragraph{Metric.}
For a task dataset, we compute the representation of each image $x$ using the model before and after adapting on that task, denoted $h_{\mathrm{before}}(x)$ and $h_{\mathrm{after}}(x)$.
We measure per-image drift by cosine shift,
\begin{equation}
\Delta(x) \;=\; 1 - \cos\!\big(h_{\mathrm{before}}(x),\,h_{\mathrm{after}}(x)\big),
\end{equation}
and report the distribution of $\Delta(x)$ over task samples using the mean and the $90$th percentile ($p90$).
We compute $\Delta(x)$ on both (i) raw image features and (ii) their $\ell_2$-normalized embeddings.

\paragraph{Results.}
Figure~\ref{fig:image_shift} shows the two tasks with the largest relative reduction in mean shift (SUN397 and Caltech101).
In both cases, the expB (blue) histogram is shifted left of expA (red), indicating smaller per-image drift under the full Tri-Affinity objective.
On SUN397~\cite{xiao2010sun397} (Task~09), expB reduces the mean shift from $0.225$ to $0.118$ ($47.5\%$ relative reduction) and reduces $p90$ from $0.311$ to $0.164$.
On Caltech101~\cite{li2004caltech101} (Task~01), the mean shift decreases from $0.239$ to $0.145$ ($39.3\%$ relative reduction) and $p90$ decreases from $0.323$ to $0.224$.
Dashed vertical lines mark the mean and dotted lines mark $p90$; both moving left indicates that the improvement is not driven only by a few outliers but reflects reduced drift across a broad portion of samples.
The embedding and feature views appear similar because the metric is cosine-based and embeddings are normalized versions of features.

\begin{figure}[t]
  \centering
  \includegraphics[width=\columnwidth]{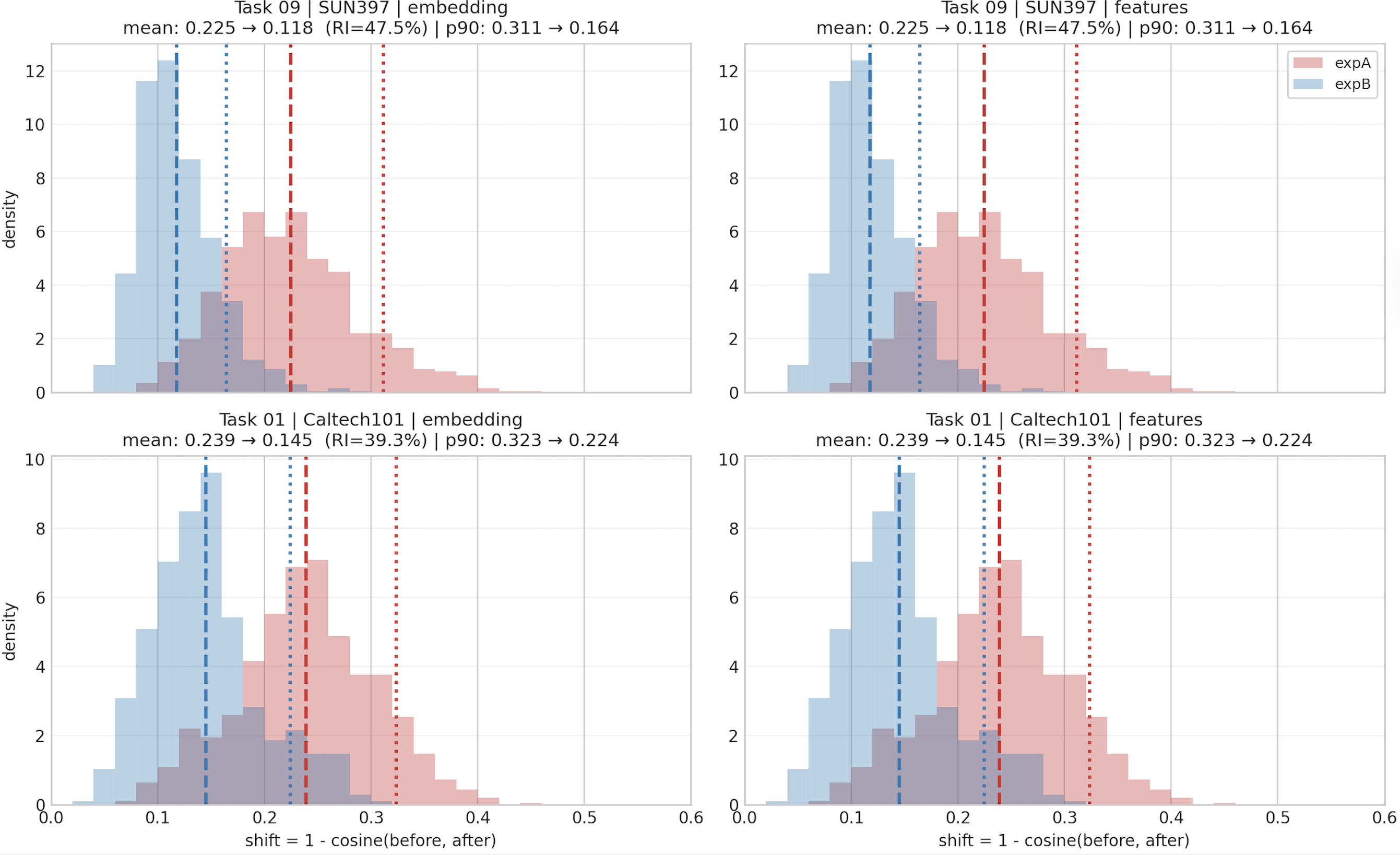}
  \caption{Per-input stability comparison between expA (cross-modal affinity only) and expB (image$+$text$+$cross-modal affinity). We plot the distribution of per-image cosine shift $1-\cos(\text{before},\text{after})$ on the two tasks (SUN397 and Caltech101), shown for both embeddings and features. Dashed vertical lines indicate the mean shift and dotted lines indicate the 90th-percentile shift ($p90$).}
  \label{fig:image_shift}
\end{figure}

\section{SACED Selection Criteria}
\label{app:saced_select}

SACED scores each prior expert on a three-signal combination: co-activation overlap (weight $0.5$), top-1 prediction agreement ($0.3$), and class-distribution similarity ($0.2$), and keeps the top $k_{\mathrm{ref}}$ candidates above the threshold.
\Cref{tab:saced_select} compares each criterion in isolation against using all prior experts and against random selection, with all rows trained under the full COMET pipeline so they are directly comparable to full COMET ($76.1$ Average, $88.1$ Last).
Co-activation alone matches full COMET on Last ($88.4$ vs.\ $88.1$) and beats random selection by $0.7$\,pp, which makes co-activation the dominant signal; the full three-signal criterion adds $1.3$\,pp Average ($76.1$ vs.\ $74.8$).
Distribution similarity alone underperforms because the signal becomes informative only after co-activation has narrowed the candidate set.
On Last, random and all-prior selection both underperform co-activation, so selectivity matters: not every prior expert contributes useful supervision.

\begin{table}[t]
\caption{SACED selection-criterion ablation on X-TAIL with reference budget $k_{\mathrm{ref}}{=}3$. The full criterion combines all three signals.}
\label{tab:saced_select}
\centering
\scriptsize
\setlength{\tabcolsep}{4pt}
\renewcommand{\arraystretch}{1.05}
\begin{tabular}{lcc}
\toprule
Selection rule & Avg & Last \\
\midrule
Full criterion (COMET)        & \textbf{76.1} & 88.1 \\
Co-activation only            & 74.8 & \textbf{88.4} \\
All prior experts             & 74.4 & 87.9 \\
Random selection              & 74.9 & 87.7 \\
Top-1 agreement only          & 74.2 & 87.0 \\
Distribution similarity only  & 67.8 & 77.2 \\
\bottomrule
\end{tabular}
\end{table}

\begin{table}[t]
\caption{Extended expert-merging ablation on X-TAIL. $\mathcal{A}/M$ = Average accuracy per stored expert. Fixed $\sigma$ replaces CMAB with a non-adaptive rule that merges whenever the candidate similarity exceeds $\sigma$.}
\label{tab:cmab_extended}
\centering
\scriptsize
\setlength{\tabcolsep}{3pt}
\renewcommand{\arraystretch}{1.05}
\begin{tabular}{lcccccc}
\toprule
Variant & $M$ & Params & Trans. & Avg & Last & $\mathcal{A}/M$ \\
\midrule
\textbf{COMET (CMAB)} & \textbf{8} & \textbf{551M} & \textbf{62.6} & \textbf{76.1} & \textbf{88.1} & \textbf{9.5} \\
Always-Expand    & 10 & 667M & 61.6 & 75.8 & 87.3 & 7.6 \\
Fixed $\sigma{=}0.9$ & 9 & 609M & 59.4 & 77.2 & 84.9 & 8.6 \\
Fixed $\sigma{=}0.7$ & 2 & 203M & 58.6 & 59.3 & 61.9 & 29.7 \\
Fixed $\sigma{=}0.5$ & 1 & 145M & 58.3 & 58.4 & 58.9 & 58.4 \\
Always-Merge     & 1 & 145M & 58.4 & 58.4 & 58.4 & 58.4 \\
Random ($p{=}0.5$) & 3 & 261M & 58.9 & 64.9 & 69.1 & 21.6 \\
\bottomrule
\end{tabular}
\end{table}

\section{Extended Expert-Merging Strategy Comparison}
\label{app:cmab_extended}

\Cref{tab:cmab_extended} reports the full seven-variant comparison underlying \Cref{tab:compact_results} (b) in the main text.
The conservative fixed threshold $\sigma = 0.9$ retains $9$ experts and edges CMAB by $1.1$\,pp Average ($77.2$ vs.\ $76.1$) but pays $-3.2$\,pp Last and $+10.5\%$ parameters; in continual learning, where Last reflects what the model can still do at the end of the stream, CMAB's trade-off favors the right axis.
Aggressive thresholds ($\sigma \in \{0.5, 0.7\}$) collapse the pool to one or two experts and lose $26$--$29$\,pp Last.
On the heterogeneous 4-task sequence, the best fixed threshold ($\sigma = 0.9$) retains all four experts but achieves only $41.9$ Average, whereas CMAB reaches $72.9$ Average and $96.7$ Last with $3$--$4$ experts.
No single fixed threshold is optimal across both sequences; the bandit adapts because its context vector reads accuracy, merge compatibility, degradation risk, capacity pressure, and expert load jointly rather than relying on similarity alone.

A merge-threshold sweep $\tau \in [0.50, 0.80]$ produces final pool sizes $M_K \in \{5, 6, 7, 9, 10\}$ that vary monotonically with $\tau$, confirming smooth control over the storage--accuracy trade-off.
EuroSAT consistently merges only at the most permissive setting; Flowers102 merges across $\tau \in [0.50, 0.65]$ at a modest cost of $-2.7$\,pp on its dedicated-expert score; per-task accuracies vary by less than $2$\,pp across $\tau$ on completed tasks.

\section{Per-Dataset Capacity Analysis}
\label{app:capacity}

\Cref{tab:capacity_gap} explains the per-dataset variation through the lens of teacher and student zero-shot accuracy.
On eight of ten X-TAIL tasks, COMET matches or exceeds the CLIP teacher's zero-shot ceiling.
Among the five datasets with positive teacher--student capacity gaps, COMET recovers $>100\%$ of the gap on four, exceeding the teacher despite a $2.6\times$ smaller student.
COMET sits below the teacher's zero-shot ceiling only on Food ($83.5$ vs.\ $84.4$) and Caltech101 ($92.9$ vs.\ $93.5$), by under one point in each case, margins within run-to-run noise.

\begin{table}[t]
\caption{Per-dataset capacity-gap analysis on X-TAIL. All accuracies use each dataset's own label space, unlike the generalized X-TAIL label space of \Cref{tab:main_results}. ZS = zero-shot. Recovery $=(\text{COMET}{-}\text{TinyCLIP ZS})/(\text{CLIP ZS}{-}\text{TinyCLIP ZS}){\times}100\%$ on positive-gap datasets; ``--'' = capacity gap is non-positive.}
\label{tab:capacity_gap}
\centering
\scriptsize
\setlength{\tabcolsep}{3pt}
\renewcommand{\arraystretch}{1.05}
\begin{tabular}{lccccc}
\toprule
Dataset & CLIP ZS & TinyCLIP ZS & Gap & COMET & Recovery \\
\midrule
Aircraft     & 22.6 & 16.9 & $+5.7$  & 83.1 & $>100\%$ \\
Caltech101   & 93.5 & 93.9 & $-0.4$  & 92.9 & -- \\
DTD          & 44.5 & 45.3 & $-0.8$  & 78.8 & -- \\
EuroSAT      & 36.1 & 50.7 & $-14.6$ & 98.5 & -- \\
Flowers102   & 65.7 & 67.9 & $-2.2$  & 95.6 & -- \\
Food         & 84.4 & 78.5 & $+5.9$  & 83.5 & $84.7\%$ \\
MNIST        & 62.6 & 38.6 & $+24.0$ & 98.8 & $251\%$ \\
OxfordPet    & 85.3 & 79.7 & $+5.6$  & 92.0 & $220\%$ \\
StanfordCars & 60.3 & 51.3 & $+9.0$  & 90.8 & $>100\%$ \\
SUN397       & 61.3 & 69.2 & $-7.9$  & 80.1 & -- \\
\bottomrule
\end{tabular}
\end{table}

\paragraph{Student backbone.}
Replacing the student with the smaller TinyCLIP-$8$M backbone causes Last to drop from $88.1$ to $61.1$ on X-TAIL, with fine-grained domains losing the most ($-72.4$\,pp on EuroSAT and $-55.2$\,pp on Flowers102), confirming that per-dataset variation tracks backbone capacity rather than method failure.
MNIST is essentially capacity-insensitive ($-1.0$\,pp), consistent with the observation that handwritten digit classification requires modest representational depth.

\paragraph{Teacher backbone.}
Replacing the CLIP ViT-B/16 teacher with the smaller CLIP ViT-B/32 (under $\tau = 0.80$) shifts Last from $86.2$ to $78.5$ ($-7.7$\,pp) and Transfer from $52.7$ to $61.5$ ($+8.8$\,pp), while $M$ decreases from $10$ to $9$.
A weaker teacher therefore lowers the upper bound on per-task accuracy but produces softer supervision that transfers more uniformly, indicating that the teacher's representational ceiling propagates to the experts.

\section{Routing Accuracy and Cross-Expert Coherence}
\label{app:routing}

Across the ten task-specific routers, the autoencoder router identifies the correct task for $95.9\%$ of test inputs at top-$1$ and $99.4\%$ at top-$2$; errors concentrate on borderline inputs where the gating margin is low ($0.621$ vs.\ $0.977$ when correctly routed).
\Cref{tab:safety_net} decomposes the predictions of the retained pool into a four-cell route$\times$predict outcome.
After merging, the router sends $2{,}069$ samples ($2.6\%$) away from their oracle expert, and the selected expert still classifies $1{,}086$ of them ($52.5\%$) correctly.
On a $7$-task subset where with- and without-SACED runs share oracle expert assignments, SACED raises misrouted-sample accuracy from $35.3\%$ to $46.1\%$ ($+10.8$\,pp) and lifts pairwise expert representation similarity (CKA) by $+0.009$, supporting the cross-expert coherence interpretation that SACED converts hard routing failures into graceful degradation.
A random-attention control confirms that the alignment SACED learns is statistically meaningful (Cohen's $d = 0.45$, $p < 0.001$) rather than noise.

\begin{table}[t]
\caption{Four-cell decomposition of test predictions on X-TAIL ($80{,}939$ samples). Route-OK = autoencoder selected the oracle expert; Predict-OK = expert produced the correct label.}
\label{tab:safety_net}
\centering
\scriptsize
\setlength{\tabcolsep}{6pt}
\renewcommand{\arraystretch}{1.05}
\begin{tabular}{lcc}
\toprule
Cell & Count & Share \\
\midrule
Route-OK $\wedge$ Predict-OK    & 66{,}664 & $82.4\%$ \\
Route-OK $\wedge$ Predict-Wrong & 12{,}206 & $15.1\%$ \\
Route-Wrong $\wedge$ Predict-OK & 1{,}086  & $1.3\%$ \\
Route-Wrong $\wedge$ Predict-Wrong & 983    & $1.2\%$ \\
\bottomrule
\end{tabular}
\end{table}

\section{Teacher Attention on Out-of-Distribution Domains}
\label{app:ood_attn}

We measure CLS-token attention concentration in the frozen teacher across three OOD domains (\Cref{tab:ood_attn}).
$C_{10}$ denotes the fraction of total CLS attention captured by the top decile of the $14 \times 14 = 196$ patches (uniform baseline $0.10$); $H$ is Shannon entropy over the patch grid in bits (uniform maximum $\log_2 196 \approx 7.61$).
Early-layer attention is diffuse on every OOD domain, but by the final layer attention concentrates $66$--$76\%$ of mass on the top decile, with entropy dropping to $5.1$--$5.7$ bits.
This is consistent with COMET's strong heterogeneous-benchmark performance: TAD does not require teacher classification correctness, only stable relative geometry, and SACED compares each expert's attention distance to the teacher with that of a prior expert, so the student never has to reproduce the teacher's attention map.

\begin{table}[t]
\caption{Frozen teacher attention statistics on three OOD domains, averaged over $N{=}50$ images per dataset. $C_{10}$: fraction of attention in the top decile of $196$ patches (uniform $0.10$). $H$: Shannon entropy in bits (uniform $\log_2 196 \approx 7.61$).}
\label{tab:ood_attn}
\centering
\scriptsize
\setlength{\tabcolsep}{4pt}
\renewcommand{\arraystretch}{1.05}
\begin{tabular}{llcccc}
\toprule
Dataset & Domain & $C_{10}$ (L0) & $H$ (L0) & $C_{10}$ (L11) & $H$ (L11) \\
\midrule
COVID-19 & Medical    & 0.338 & 7.11 & \textbf{0.663} & 5.71 \\
SEM      & Microscopy & 0.280 & 7.28 & \textbf{0.709} & 5.48 \\
NCT-CRC  & Histopath  & 0.228 & 7.39 & \textbf{0.760} & 5.08 \\
\bottomrule
\end{tabular}
\end{table}

\section{Held-out Zero-Shot Generalization}
\label{app:heldout}

After completing the full 10-task X-TAIL stream, the same model retains zero-shot accuracy on datasets it never saw during continual learning: $45.6$ on CIFAR-100, $71.0$ on CIFAR-10, and $60.6$ on CUB-200.
Because TAD anchors the student's affinity geometry to the frozen teacher rather than to any specific downstream label space, the student inherits a broadly transferable representation that survives the full continual-learning sequence.

\section{Detailed Computational Cost}
\label{app:cost_extended}

\Cref{tab:cost_train_full} reports per-task training cost in detail.
Among CLIP-based methods, COMET trains $2.6\times$ fewer parameters per task than full-CLIP fine-tuning baselines (LwF, WiSE-FT, ZSCL) and uses below-WiSE-FT VRAM because only the current expert receives gradients while the teacher and at most $k_{\mathrm{ref}} \le 2$ frozen reference experts are forward-only.
PEFT methods (MoE-Adapters, LADA) retain a smaller VRAM footprint because their trainable parameter count is one to two orders of magnitude smaller, but they still require the full frozen CLIP backbone in memory for forward passes.

\begin{table}[t]
\caption{Per-task training cost on a single NVIDIA A6000 ($48$\,GB), batch size $64$, $1{,}000$ iterations. ``Trainable'' counts only parameters that receive gradients.}
\label{tab:cost_train_full}
\centering
\scriptsize
\setlength{\tabcolsep}{4pt}
\renewcommand{\arraystretch}{1.05}
\begin{tabular}{lccc}
\toprule
Method & Trainable & Peak VRAM & Backbone in Memory \\
\midrule
LwF          & $\sim 150$M           & $\sim 4.1$\,GB & Full CLIP (trainable) \\
WiSE-FT      & $\sim 150$M           & $\sim 2.9$\,GB & Full CLIP (trainable) \\
ZSCL         & $\sim 150$M $(+150$M ref$)$ & $\sim 4.1$\,GB & Full CLIP $+$ frozen ref \\
MoE-Adapters & $2$--$5$M             & $\sim 1.1$\,GB & Frozen CLIP \\
LADA         & $1$--$3$M             & $\sim 0.9$\,GB & Frozen CLIP \\
\textbf{COMET (Ours)} & $\sim 58$M   & $\sim 2.7$\,GB & Frozen teacher \\
\bottomrule
\end{tabular}
\end{table}

\paragraph{Storage.}
With $M = 8$ retained experts after the X-TAIL stream, COMET stores roughly $8 \times 58\,{\rm M} + 87\,{\rm M} \approx 551$\,M parameters (eight experts and the teacher's vision encoder); the eight routers add about $1.6$\,M.
Because at most one expert is active per input at inference, this stored--active ratio of $3.8\times$ is conservative by mixture-of-experts standards (Switch-Base: $\sim 30\times$, Mixtral: $\sim 3.6\times$) and does not enter the deployment compute budget.

\section{16-Shot Results}
\label{app:main_16}

We re-run X-TAIL under a 16-shot per-class budget per task, holding all other hyperparameters at their defaults.
\Cref{tab:main_16} reports the per-task breakdown across all ten datasets and the three continual-learning metrics.
Across the ten tasks, COMET's Average is $72.6$ ($+9.0$\,pp over no-distillation) and Last is $83.3$ ($+10.0$\,pp over no-distillation), preserving the qualitative ordering of the full-shot table in the main text: TAD and SACED retain their additive contribution at low data budgets.
The same trend holds on the heterogeneous 4-task sequence, with COMET continuing to lead all baselines on Average and Last.

\begin{table}[H]
\caption{Comparison of different continual learning methods on X-TAIL~\citep{xu2024x-tail} 16-shot for each task in terms of \textit{Transfer}, \textit{Average}, and \textit{Last} scores (\%). The best results are highlighted with \textbf{bold} style.}
\label{tab:main_16}
\centering
\resizebox{\textwidth}{!}{%
\begin{tabular}{lccccccccccc}
\toprule
\textbf{Method} & \rotatebox{45}{Aircraft} & \rotatebox{45}{Caltech101} & \rotatebox{45}{DTD} & \rotatebox{45}{EuroSAT} & \rotatebox{45}{Flowers} & \rotatebox{45}{Food} & \rotatebox{45}{MNIST} & \rotatebox{45}{Pets} & \rotatebox{45}{Cars} & \rotatebox{45}{Sun397} & Average \\
\midrule
\multicolumn{12}{l}{\texttt{TinyCLIP ViT-39M/16@224px}} \\
Zero-shot & 17.0 & 75.6 & 39.0 & 40.6 & 65.0 & 77.4 & 22.0 & 78.9 & 47.8 & 67.2 & 53.1 \\
\multicolumn{12}{l}{\texttt{CLIP ViT-B/16@224px}} \\
Zero-shot & 23.8 & 74.3 & 36.4 & 37.4 & 64.1 & 83.4 & 43.9 & 87.8 & 65.5 & 60.8 & 57.7 \\
\midrule
\multicolumn{12}{l}{\textbf{Transfer}}\\
LwF & -- & 64.8 & 29.1 & 7.6 & 49.7 & 72.0 & 35.6 & 64.0 & 29.0 & 53.7 & 45.1 \\
WiSE-FT & -- & 68.3 & 30.8 & 17.3 & 59.2 & 82.5 & 41.4 & 74.9 & 47.8 & 59.6 & 53.5 \\
ZSCL & -- & 71.7 & 32.6 & 15.1 & 63.0 & 85.0 & 29.5 & 83.6 & 59.4 & 63.2 & 55.9 \\
MoE-Adapters & -- & 76.2 & 30.4 & 37.6 & 65.4 & \textbf{87.9} & \textbf{53.0} & \textbf{88.4} & 63.0 & 59.1 & \textbf{62.3} \\
LADA & -- & 74.9 & 35.9 & 36.5 & \textbf{66.1} & 84.2 & 37.9 & 88.0 & \textbf{65.3} & 60.1 & 60.9 \\
COMET (Ours) & -- & \textbf{85.6} & \textbf{45.5} & \textbf{44.2} & 65.1 & 79.0 & 35.9 & 77.7 & 47.3 & \textbf{67.6} & 60.9 \\
\midrule
\multicolumn{12}{l}{\textbf{Average}}\\
LwF & 18.2 & 80.9 & 42.3 & 34.0 & 48.9 & 80.7 & 59.6 & 71.1 & 35.8 & 56.3 & 52.8 \\
WiSE-FT & 27.6 & 80.1 & 43.3 & 38.5 & 63.3 & 86.5 & 62.3 & 79.3 & 54.6 & 61.5 & 59.7 \\
ZSCL & 40.1 & 78.2 & 49.9 & 40.0 & 73.5 & \textbf{87.7} & 56.2 & 86.2 & 64.6 & 64.8 & 64.1 \\
MoE-Adapters & 42.8 & 77.8 & 33.1 & 54.5 & 67.4 & 87.1 & 58.7 & 87.7 & 62.1 & 59.5 & 63.1 \\
Primal-RAIL & 44.9 & 89.5 & 55.6 & 68.9 & 83.5 & 85.3 & \textbf{64.9} & 88.1 & 69.1 & 64.9 & 71.5 \\
LADA & 48.8 & 92.0 & 60.8 & \textbf{71.3} & 84.5 & 84.1 & 60.3 & \textbf{88.9} & \textbf{69.2} & 61.6 & 72.1 \\
COMET (Ours) & \textbf{66.7} & \textbf{96.2} & \textbf{68.5} & 68.5 & \textbf{84.7} & 78.8 & 59.4 & 80.1 & 55.2 & \textbf{68.3} & \textbf{72.6} \\
\midrule
\multicolumn{12}{l}{\textbf{Last}}\\
LwF & 10.4 & 74.6 & 47.0 & 23.2 & 48.7 & 89.1 & 93.1 & 86.8 & 56.1 & \textbf{79.3} & 60.8 \\
WiSE-FT & 29.8 & 78.5 & 46.5 & 31.6 & 70.9 & 90.3 & 90.1 & 87.5 & 79.8 & 78.8 & 68.4 \\
ZSCL & 34.9 & 78.8 & 53.1 & 33.4 & 80.7 & \textbf{90.7} & 95.7 & 92.2 & 84.7 & 78.6 & 72.3 \\
MoE-Adapters & 42.8 & 78.0 & 33.8 & 61.8 & 68.8 & 86.4 & 67.2 & 86.1 & 58.5 & 63.2 & 64.7 \\
Primal-RAIL & 44.5 & 94.5 & 67.4 & 86.1 & 96.4 & 86.9 & 91.6 & \textbf{91.6} & 80.8 & 75.7 & 81.6 \\
LADA & 49.3 & 95.1 & 68.4 & \textbf{86.2} & 96.8 & 84.5 & 93.8 & 91.3 & 84.8 & 74.9 & 82.5 \\
COMET (Ours) & \textbf{66.5} & \textbf{97.4} & \textbf{73.9} & 76.9 & \textbf{97.8} & 78.5 & \textbf{94.8} & 85.5 & \textbf{86.8} & 74.8 & \textbf{83.3} \\
\bottomrule
\end{tabular}%
}
\end{table}


\clearpage
\section*{NeurIPS Paper Checklist}

\begin{enumerate}

\item {\bf Claims}
    \item[] Question: Do the main claims made in the abstract and introduction accurately reflect the paper's contributions and scope?
    \item[] Answer: \answerYes{}
    \item[] Justification: Yes; the abstract and Introduction state the three targeted failure modes and the COMET contributions, and the claims are supported by the method sections and experiments in Sections~\ref{sec:tri-affinity}--\ref{sec:routing} and Tables~\ref{tab:main_results}--\ref{tab:compact_results}.

\item {\bf Limitations}
    \item[] Question: Does the paper discuss the limitations of the work performed by the authors?
    \item[] Answer: \answerYes{}
    \item[] Justification: Yes; limitations of the expert-growth argument are discussed in Appendix~\ref{app:growth_proof}, including the dependence on stream novelty and decision errors, and computational/storage trade-offs are discussed in the cost analysis and Appendix~\ref{app:cost_extended}.

\item {\bf Theory assumptions and proofs}
    \item[] Question: For each theoretical result, does the paper provide the full set of assumptions and a complete (and correct) proof?
    \item[] Answer: \answerYes{}
    \item[] Justification: Yes; the sub-linear expert-growth result states its assumptions and proof in Appendix~\ref{app:growth_proof}, including Theorem~\ref{thm:sublinear_growth} and Corollary~\ref{cor:sublinear_with_errors}.

    \item {\bf Experimental result reproducibility}
    \item[] Question: Does the paper fully disclose all the information needed to reproduce the main experimental results of the paper to the extent that it affects the main claims and/or conclusions of the paper (regardless of whether the code and data are provided or not)?
    \item[] Answer: \answerYes{}
    \item[] Justification: Yes; the experimental setup, datasets, metrics, model choices, training schedule, routing procedure, and hyperparameters are described in Section~\ref{sec:cmab}, Section~\ref{sec:routing}, the Experiments section, Algorithm~\ref{alg:comet}, and Appendix~\ref{app:impl}.

\item {\bf Open access to data and code}
    \item[] Question: Does the paper provide open access to the data and code, with sufficient instructions to faithfully reproduce the main experimental results, as described in supplemental material?
    \item[] Answer: \answerYes{}
    \item[] Justification: We plan to release the code, reproduction scripts, configuration files, and setup instructions in an open-source GitHub repository by the camera-ready version. Reproducibility is also supported through the public datasets, full algorithmic description, implementation details, hyperparameters, and compute information in Algorithm~\ref{alg:comet} and Appendix~\ref{app:impl}.

\item {\bf Experimental setting/details}
    \item[] Question: Does the paper specify all the training and test details (e.g., data splits, hyperparameters, how they were chosen, type of optimizer) necessary to understand the results?
    \item[] Answer: \answerYes{}
    \item[] Justification: Yes; the benchmarks, data sequences, evaluation metrics, model backbones, optimizer, learning rate, batch size, reference pool, loss weights, CMAB settings, and routing rule are specified in the Experiments section and Appendix~\ref{app:impl}.

\item {\bf Experiment statistical significance}
    \item[] Question: Does the paper report error bars suitably and correctly defined or other appropriate information about the statistical significance of the experiments?
    \item[] Answer: \answerNo{}
    \item[] Justification: No; the current manuscript reports task-order robustness and sensitivity analyses in Appendices~\ref{app:order}--\ref{app:sensitivity}, but does not provide seed-level error bars or formal significance tests for all main experimental tables.

\item {\bf Experiments compute resources}
    \item[] Question: For each experiment, does the paper provide sufficient information on the computer resources (type of compute workers, memory, time of execution) needed to reproduce the experiments?
    \item[] Answer: \answerYes{}
    \item[] Justification: Yes; Section~\emph{Analysis of Computational Cost} and Appendix~\ref{app:cost_extended} report GPU type, VRAM, FLOPs, latency, throughput, per-task training time, and total training cost.
    
\item {\bf Code of ethics}
    \item[] Question: Does the research conducted in the paper conform, in every respect, with the NeurIPS Code of Ethics \url{https://neurips.cc/public/EthicsGuidelines}?
    \item[] Answer: \answerYes{}
    \item[] Justification: Yes; the work uses public datasets and existing pretrained vision-language models, involves no human-subject experiments, and is evaluated as a continual-learning method without collecting private or sensitive data.

\item {\bf Broader impacts}
    \item[] Question: Does the paper discuss both potential positive societal impacts and negative societal impacts of the work performed?
    \item[] Answer: \answerNo{}
    \item[] Justification: No; the current manuscript focuses on the method and empirical evaluation and does not yet include a dedicated discussion of positive and negative societal impacts.
    
\item {\bf Safeguards}
    \item[] Question: Does the paper describe safeguards that have been put in place for responsible release of data or models that have a high risk for misuse (e.g., pre-trained language models, image generators, or scraped datasets)?
    \item[] Answer: \answerNA{}
    \item[] Justification: Not applicable; the paper does not introduce or release high-risk assets such as generative models, pretrained language models, scraped datasets, or systems intended for direct deployment.

\item {\bf Licenses for existing assets}
    \item[] Question: Are the creators or original owners of assets (e.g., code, data, models), used in the paper, properly credited and are the license and terms of use explicitly mentioned and properly respected?
    \item[] Answer: \answerNo{}
    \item[] Justification: No; existing datasets and pretrained models are credited through citations in the Related Work and Experiments sections, but the current manuscript does not yet include a complete license table. The planned open-source GitHub repository for the camera-ready version will include the licenses, terms of use, and attribution information for all external assets used in the experiments.

\item {\bf New assets}
    \item[] Question: Are new assets introduced in the paper well documented and is the documentation provided alongside the assets?
    \item[] Answer: \answerNA{}
    \item[] Justification: Not applicable; the paper proposes a method and does not introduce a new dataset, benchmark, model release, or other new asset in the current submission.

\item {\bf Crowdsourcing and research with human subjects}
    \item[] Question: For crowdsourcing experiments and research with human subjects, does the paper include the full text of instructions given to participants and screenshots, if applicable, as well as details about compensation (if any)? 
    \item[] Answer: \answerNA{}
    \item[] Justification: Not applicable; the paper does not involve crowdsourcing, user studies, or research with human subjects.

\item {\bf Institutional review board (IRB) approvals or equivalent for research with human subjects}
    \item[] Question: Does the paper describe potential risks incurred by study participants, whether such risks were disclosed to the subjects, and whether Institutional Review Board (IRB) approvals (or an equivalent approval/review based on the requirements of your country or institution) were obtained?
    \item[] Answer: \answerNA{}
    \item[] Justification: Not applicable; the paper does not involve crowdsourcing or human-subject research, so IRB or equivalent review is not required.

\item {\bf Declaration of LLM usage}
    \item[] Question: Does the paper describe the usage of LLMs if it is an important, original, or non-standard component of the core methods in this research? Note that if the LLM is used only for writing, editing, or formatting purposes and does \emph{not} impact the core methodology, scientific rigor, or originality of the research, declaration is not required.
    \item[] Answer: \answerNA{}
    \item[] Justification: Not applicable; LLMs are not used as an important, original, or non-standard component of the proposed method, experiments, or scientific claims.

\end{enumerate}

\end{document}